\documentclass{article}

    \PassOptionsToPackage{numbers}{natbib}

    \usepackage[preprint]{neurips_2022}

\usepackage{microtype}
\usepackage{graphicx}
\usepackage{subcaption}
\usepackage{booktabs} 

\usepackage[utf8]{inputenc} 
\usepackage[T1]{fontenc}    
\usepackage{hyperref}       
\usepackage{url}            
\usepackage{booktabs}       
\usepackage{amsfonts}       
\usepackage{nicefrac}       
\usepackage{microtype}      
\usepackage{xcolor}         
\usepackage{microtype}      
\usepackage{xcolor}         
\usepackage{amsmath}
\usepackage{amsthm}
\usepackage{bm}
\usepackage{amssymb}
\usepackage{mathtools}
\usepackage{subcaption,graphicx}
\usepackage{caption}
\usepackage{wrapfig}
\usepackage{multirow}
\usepackage{graphics}
\theoremstyle{plain}
\newtheorem{theorem}{Theorem}[section]

\newtheorem{lemma}[theorem]{Lemma}
\newtheorem{corollary}[theorem]{Corollary}
\theoremstyle{definition}

\theoremstyle{remark}

\usepackage[ruled, vlined, nofillcomment, linesnumbered, noline]{algorithm2e}

\usepackage{pgfplots}
\usepackage{subcaption}
\usepackage{palettes}
\usepackage[notheorems]{eda}
\usetikzlibrary{patterns,positioning}
\pgfkeys{/pgfplots/tuftelike/.style={
  semithick,
  tick style={major tick length=4pt,semithick,black},
  separate axis lines,
  axis x line*=bottom,
  axis x line shift=10pt,
  xlabel shift=5pt,
  axis y line*=left,
  axis y line shift=10pt,
  ylabel shift=5pt}}

\usepackage{amsmath,amsfonts,bm}

\def\reals{\mathbb{R}}
\DeclarePairedDelimiter{\norm}{\lVert}{\rVert}
\newtheorem*{theorem*}{Theorem}
\newtheorem*{lemma*}{Lemma}

\def\eqref#1{equation~\ref{#1}}

\def\1{\bm{1}}

\def\eps{{\epsilon}}

\DeclareMathAlphabet{\mathsfit}{\encodingdefault}{\sfdefault}{m}{sl}
\SetMathAlphabet{\mathsfit}{bold}{\encodingdefault}{\sfdefault}{bx}{n}

\DeclareMathOperator*{\argmax}{arg\,max}
\DeclareMathOperator*{\argmin}{arg\,min}

\newcommand{\oururl}{Anonymized code is available to the reviewers and will be made publicly available after publication.}

\title{Lottery Tickets with Nonzero Biases}

\makeatletter
\newcommand{\printfnsymbol}[1]{%
  \textsuperscript{\@fnsymbol{#1}}%
}
\makeatother

\author{%
Jonas Fischer\thanks{These authors contributed equally to this work.} \\
   Max Planck Institute for Informatics, Saarbrücken, Germany\\
   Saarbrücken, Germany \\
   \texttt{fischer@mpi-inf.mpg.de} \\
   \And
  Advait Gadhikar\printfnsymbol{1} \\
  CISPA Helmholtz Center for Information Security\\
  Saarbrücken, Germany\\
  \texttt{advait.gadhikar@cispa.de} \\
   \And
   Rebekka Burkholz \\
   CISPA Helmholtz Center for Information Security\\
   Saarbrücken, Germany \\
   \texttt{burkholz@cispa.de} \\
   }

\begin{document}

\maketitle

\begin{abstract}
  The strong lottery ticket hypothesis holds the promise that pruning randomly initialized deep neural networks could offer a computationally efficient alternative to deep learning with stochastic gradient descent. 
Common parameter initialization schemes and existence proofs, however, are focused on networks with zero biases, thus foregoing the potential universal approximation property of pruning.
To fill this gap, we extend multiple initialization schemes and existence proofs to nonzero biases, including explicit 'looks-linear' approaches for ReLU activation functions.
These do not only enable truly orthogonal parameter initialization but also reduce potential pruning errors.
In experiments on standard benchmark data, we further highlight the practical benefits of nonzero bias initialization schemes, and present theoretically inspired extensions for state-of-the-art strong lottery ticket pruning.
\end{abstract}

\section{Introduction}\label{sec:introduction}

Challenging tasks across different domains, from protein structure prediction for drug development to detection in complex scenes for self driving cars, have recently been solved through deep neural networks (NNs).
This success, however, is due to heavy overparametrization and comes at the expense of large amounts of computational resources that these models require to be trained and to be deployed.
While training small NNs from scratch commonly fails, the lottery ticket hypothesis (LTH) conjectured by \citet{frankle2019lottery} bears a potential solution.
The LTH states that within a large, randomly initialized NN there exists a well trainable, much smaller subnetwork, or 'ticket', which can be identified by pruning the large NN.
Thus, both training and deployment becomes computationally much cheaper at the expense of the pruning algorithm. 
Even more promising is the conjecture of the existence of 'strong lottery tickets' (SLTs) by \citet{ramanujan2019whats}, which are subnetworks of randomly initialized NNs that do \emph{not} require any further training, as expensive training might become obsolete. 
This SLT hypothesis (SLTH) was later proven for networks without biases \citep{malach2020proving,pensia2020optimal,orseau2020logarithmic}.

Whereas standard intialization schemes -- and hence SLTs -- use zero biases, most successfully trained NN architectures have nonzero biases.
Such bias terms are important to equip NNs with the universal approximation property~\citep{scarselli1998universal}. 
We also prove formally that NNs without biases fail to achieve this in general.
To enable successful training by pruning alone, we generalize common initialization schemes, including the 'looks linear' orthogonal initialization, to nonzero biases.
We show that a signal that computes the output of a function or gradients can propagate through such nonzero bias initialized networks. 
This means they are trainable and gradient-based scoring functions, which are typically used by lottery ticket pruners, are computable.
Moreover, we formally prove the SLTH in this setting and provide empirical evidence that our thoery derives realistic conditions, in which lottery ticket pruning is feasible.

For the discovery of SLTs in practice, \citet{ramanujan2019whats} proposed \texttt{edge-popup}, to the best of our knowledge the only algorithm capable of doing so. 
Their proposal is, however, not suited to recover bias parameters and finds only relatively dense tickets. 
We here extend their approach in multiple ways to recover strong tickets that include bias parameters and that are sparser than the ones obtained by vanilla \texttt{edge-popup}.
In particular, we extend the popup scores to bias terms, and slowly anneal the sparsity of the network to the desired target sparsity.
In addition, our LTH proof conjetures that a rescaling of the parameters is necessary to find SLTs of high sparsity. 
We propose an efficient optimization procedure to find such a good scaling factor in each epoch.

In a synthetic data study, we show that SLTs recovered by \texttt{edge-popup} from NNs with nonzero bias initialization outperform tickets found in networks initialized with zero bias. 
Furthermore, we show that on this data, \texttt{edge-popup}  with rescaling finds much sparser tickets of higher quality.
While these results are encouraging, we discuss that on image data, nonzero biases might play a less prominent role for the current generation of SLT pruning algorithms, which identify subnetworks of suboptimal sparsity with few bias parameters in the recovered ticket~\cite{fischer2022planting}.
We anticipate that our generalization of strong tickets and initialization schemes could pave the way for the next generation of algorithms for the discovery of parameter efficient subnetworks at initialization.

Our contributions are (i) we provide formal evidence that SLTs of zero bias initialized network are not universal function approximators, (ii) we generalize the SLTH to networks with potentially nonzero initial biases, which enables pruning alone to achieve the universal approximation property, (iii) we extend standard initialization schemes to nonzero biases and prove trainability and the existence of SLTs in this setting, and (iv) we enable pruning algorithms for SLTs, to find tickets with biases, which are hence capable of universal approximation.


\subsection{Related work}

The lottery ticket hypothesis \citep{frankle2019lottery} has spurred the development of neural network pruning algorithms that either prune before \citep{grasp,snip,synflow, snipit}, during \citep{frankle2019lottery,srinivas2016gendropout,earlybird,orthoRepair,weightcor,liu2021:finetune,weightelim,LTreg,elasticLTH,sanity,sanity2}, or after training \citep{sigmoidl0,lecun1990optimal,hassibi1992second,dong2017surgeon,li2017pruneconv,molchanov2017pruneinf,validateManifold}.
Their main objective is to identify a subnetwork of a randomly initialized neural network that can be trained to achieve a similar performance as the trained full network. 
\citet{supermask} and \citet{ramanujan2019whats} postulated an even stronger hypothesis, as they realized that the randomly initialized neural network contains so called 'strong' lottery tickets that do not require further training after pruning.
\citet{malach2020proving,pensia2020optimal,orseau2020logarithmic} proved their existence by deriving realistic lower bounds on the width of the original randomly initialized neural network that contains a lottery ticket with a given probability.

However, the proposed pruning algorithm for SLTs, i.e., edge-popup \citep{ramanujan2019whats}, as well as the existence proofs only handle neural network architectures with zero biases.
A reason might be that very large neural networks can compensate for missing biases in the studied application, i.e. image classification. 
Yet, this compensation usually require a higher number of weight parameters but state-of-the-art algorithms and proofs for SLTs do not cover highly sparse tickets \citep{fischer2022planting}.
Another reason might be that most neural network initialization schemes propose zero biases. 
Exceptions include a data dependent choice of biases \citep{batchMFT} and a random scheme that does not try to prevent exploding or vanishing gradients in deep neural networks \citep{complLinReg}.
The recent trend in the search for weak lottery tickets towards rewinding parameters to values obtained early during training \citep{rewind} also results in lottery ticket initialization with nonzero biases.
None of these nonzero bias initialization schemes are designed in support of the existence of SLTs.
We fill this gap with this work.

\subsection{Notation}
Let $f(x)$ denote a bounded function, without loss of generality $f: {[-1,1]}^{n_0} \rightarrow {[-1,1]}^{n_L}$, that is parameterized as a deep neural network with architecture  $\bar{n} = [n_0, n_1, ..., n_L]$, i.e., depth $L$ and widths $n_l$ for layers $l = 0, ..., L$ with ReLU activation function $\phi(x):= \max(x,0)$.
It maps an input vector $\bm{x}^{(0)}$ to neurons $x^{(l)}_i$ as:
\begin{align*}\label{eq:DNN}
 \bm{x}^{(l)} = \phi\left(\bm{h}^{(l)} \right), \qquad
 \bm{h}^{(l)} = \bm{W}^{(l)} \bm{x}^{(l-1)} + \bm{b}^{(l)}, \qquad
 \bm{W}^{(l)} \in \mathbb{R}^{n_{l} \times n_{l-1}},
 \qquad\bm{b}^{(l)} \in \mathbb{R}^{n_l} 
\end{align*}
where $\bm{h}^{(l)}$ is the pre-activation, $\bm{W}^{(l)}$ is the weight matrix, and $\bm{b}^{(l)}$ is the bias vector of layer $l$.
For convenience, the parameters of the network are subsumed in a vector $\bm{\theta} := \left(\left(\bm{W}^{(l)}, \bm{b}^{(l)}\right)\right)^{L}_{l=1}$.
We also write $f(x \mid \bm{\theta})$ to emphasize the dependence of $f$ on its parameters $\bm{\theta}$.

The supremum norm of any function $g$ is defined with respect to the same domain $||g||_{\infty} := \sup_{x \in {[-1,1]}^{n_0}} ||g||_2$. 

Assume furthermore that a ticket $f_{\epsilon}$ can be obtained by pruning a large mother network $f_0$, which we indicate by writing $f_{\epsilon} \subset f_{0}$. 
The sparsity level $\rho$ of $f_{\epsilon}$ is then defined as the fraction of nonzero weights that remain after pruning, i.e., $\rho = \left(\sum_l ||\bm{W}^{(l)}_{\epsilon}||_0 \right)/\left(\sum_l  ||\bm{W}^{(l)}_{0}||_0 \right)$, where $||\cdot||_{0}$ denotes the $l_0$-norm, which counts the number of nonzero elements in a vector or matrix.
Another important quantity that influences the existence probability of lottery tickets is the in-degree of a node $i$ in layer $l$ of the target $f$, which we define as the number of nonzero connections of a neuron to the previous layer plus $1$ if the bias is nonzero, i.e., $k^{(l)}_i := ||\bm{W}^{(l)}_{i,:}||_{0} + ||b^{(l)}_{i}||_{0}$, where $\bm{W}^{(l)}_{i,:}$ is the $i$-th row of $\bm{W}^{(l)}$.
The maximum degree of all neurons in layer $l$ is denoted as $k_{l,\text{max}}$.
In the formulation of theorems, we make use of the universal constant $C$ that can attain different values. 

\section{Motivation: Neural Networks with Zero Biases are Not Universal Function Approximators}

\begin{wrapfigure}[18]{r}{5.5cm}
    \centering
    \includegraphics[width=5cm]{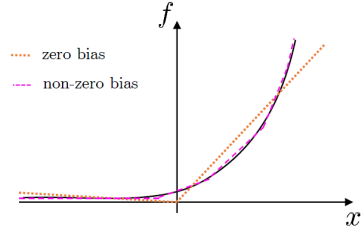}
    \caption{\textit{Function approximations.} Depicted are approximations of an exponential function (solid black) by sparse NNs with nonzero bias (dashed pink), which are piecewise linear functions, and zero bias (dotted orange), which are two linear pieces through the origin, which exemplify Lem.~\ref{lem:networkfact}.}
    \label{fig:lemma_exp}
\end{wrapfigure}

Why do we need to initialize nonzero biases?
With nonzero biases, neural networks have the universal approximation property~\citep{scarselli1998universal}. 
Thus, for a given $\epsilon > 0$ and arbitrary continuous function $g$, a large enough neural network can approximate $g$ up to error $\epsilon$. 
Standard neural networks and also weak lottery tickets are able to learn nonzero biases even from zero initialization. However, this is not the case with strong lottery tickets. Strong lottery tickets rely on pruning alone to obtain a model with high performance. Nonzero biases need to be available from the initialization for them to be universal approximators with ReLU activations.
The following Lemma captures the problem with zero bias networks, which is a factorization of univariate ReLU networks of \textit{arbitrary depth and width}.

\begin{lemma}\label{lem:networkfact}
  Univariate neural networks with ReLU activations  $f:\reals\rightarrow\reals^{n_L}$ of arbitrary depth $L$ and layer widths $n_1,\ldots,n_L$, and without biases, represent a function $f(x) = \bm{W}_{+} \phi(x) + \bm{W}_{-} \phi(-x), \bm{W}_{+},\bm{W}_{-}\in\reals^{n_L \times 1}$.
\end{lemma}
This factorization shows that ReLU networks without biases are greatly limited in terms of the functions they can express, which we visualize for an example in Fig.~\ref{fig:lemma_exp}.
From this Lemma we derive the following insight about neural networks without biases -- and hence strong lottery tickets -- which forms the motivation of our work.

\begin{corollary}\label{thm:nonUniversalApprox}
Neural networks with ReLU activations and without biases are not universal function approximators. 
\end{corollary}
\begin{proof}
We prove by contradiction.
Assume that such a network is a universal function approximator, i.e. for any function $g$ and any $\epsilon$ there exists a network that can approximate $g$ up to error $\epsilon$.

Let us try to approximate the constant function $g(x) = 0.5$ on the domain $[-1,1]$ with a neural network without biases.
From Lemma \ref{lem:networkfact}, we know that an univariate ReLU network without biases represents a function $f(x) = w_{+} \phi(x) + w_{-} \phi(-x)$ with two parameters $w_{+}, w_{-} \in \mathbb{R}$.
The minimum mean squared error with respect to $g(x)$ that $f$ can achieve is $\int^1_{-1} (g(x) - f(x))^2 dx = 1/8$ for $w_{+} = 3/4$ and $w_{-} = 3/4$ (see Appendix \ref{app:counterex_universal}).
Thus for any $\epsilon < 1/8$, $f(x)$ fails to approximate $g(x)$ up to error $\epsilon$, which is a contradiction.
Note that a ReLU network with nonzero biases can represent the function $g(x) = 0.5$ perfectly.
For instance, a network of depth $L=1$ with one neuron in the intermediary layer is sufficient, as $0.5 = \phi(0.5)$.

Although $g(x)=0.5$ is an obviously simple function, it already serves as a counterexample and hence suffices to prove the theorem.
Networks without a bias need to explicitly model an argument independent offset -- or bias, which is a hard task.
It is easy to see from Lemma \ref{lem:networkfact} that networks with zero bias usually fail to guarantee universal approximation for univariate non-linear functions. We provide $g(x)=e^x$ as an example in Appendix \ref{app:counterex_universal}.
\end{proof}
This theorem provides the theoretical motivation for why we need nonzero bias initializations for well performing SLTs, and why it is necessary to include bias terms when pruning for SLTs.


\section{Initialization of Nonzero Biases}
Common initialization schemes (e.g. \cite{HeInit,GlorotInit,dyniso}) set all biases to zero, while network weights are drawn randomly to obtain parameter diversity. 
Proofs of the strong lottery ticket hypothesis have focused on this setting, thereby foregoing the universal approximation property of deep neural networks \citep{scarselli1998universal}, since pruning alone can only recover the zero-initialized biases.

Here, we propose the initialization of nonzero biases, which then become subject to pruning in addition to the network weights.
How should these biases be initialized?
A good approach has to fulfill two essential criteria. 
a) The randomly initialized neural network needs to be trainable by SGD. 
This property is also critical for most pruning algorithms, as they are inspired by SGD and define pruning scores based on gradients.
b) The randomly initialized neural network should contain lottery tickets with high probability.

Before we can answer how to initialize biases, we first have to face a different issue pertaining strong lottery tickets.
Standard initialization approaches, like He \citep{HeInit} or Glorot \citep{GlorotInit} initialization, achieve trainability by ensuring that the output of a deep neural network is contained within a reasonable range, thus rendering the computations of gradients numerically feasible. 
Network weights are commonly initialized according to a distribution with variance $\sigma^2$ that is inverse proportional to the number of neurons $n$ in a layer, $\sigma^2 \propto 1/n$. 
In consequence, after pruning a high percentage of these weights, the network output is heavily down-scaled, which needs to be compensated by up-scaling the output, as also discovered experimentally by \cite{ramanujan2019whats} and mentioned by \cite{malach2020proving}.

\subsection{Output Scaling}\label{sec:outscaling}
For ReLU networks with zero biases, the appropriate output scaling after or during pruning is straight forward to compute, as networks of depth $L$ are $L$-homogeneous in the network parameters. 
Multiplying each parameter with the same scalar $\sigma$ leads to a scaling factor of $\sigma^L$:  $f(x \mid\sigma \theta) = \sigma^L f(x\mid\theta)$. 
This holds no longer true for nonzero biases. 
The following observation helps us to develop a notion that is similar to homogeneity for networks with nonzero biases. 
\begin{lemma}\label{thm:scaling}
Let  $h\left(\bm{\theta_0}, \bm{\sigma}\right)$ denote a transformation of the parameters $\bm{\theta_0}$ of the deep neural network $f_0$, where each weight is multiplied by a scalar $\sigma_l$, i.e., $h^{(l)}_{ij}(w^{(l)}_{0,ij}) = \sigma_l w^{(l)}_{0,ij}$, and each bias is transformed to $h^{(l)}_{i} (b^{(l)}_{0,i}) = \prod^l_{m=1}\sigma_m b^{(l)}_{0,i}$.
Then, we have 
$f\left(x \mid h(\bm{\theta_0}, \bm{\sigma})\right) = \prod^L_{l=1} \sigma_l f(x \mid \bm{\theta_0})$.
\end{lemma}
Lemma~\ref{thm:scaling} suggests that if we scale each weight by a factor $\sigma_{w,l}$, scaling the corresponding biases by a factor $\sigma_{b,l} = \prod^l_{m=1}\sigma_{w,m}$ would result in the same network $f$ without scaling of parameters. 
We only have to correct the output by dividing it with a factor $\prod^L_{l=1} \sigma_{w,l}$. 
From this observation, we directly derive our initialization proposal, as it suggests an equivalence (irrespective of scaling) between initialising parameters in $\theta_i \in U[0,1]$ and our more realistic setting. 
Concretely, we propose to replace $b^{(l)}_{i} = 0$ by $b^{(l)}_{i} \sim U{([-\prod^l_{k=1}\sigma_{w,k}, \prod^l_{k=1}\sigma_{w,k}])}$ or $b^{(l)}_{i} \sim \mathcal{N}\left(0, \prod^l_{m=1}\sigma_{w,m}\right)$, respectively, when the weights are $w^{(l)}_{ij} \sim  U\left([-\sigma_{w,l}, \sigma_{w,l}]\right)$  or $w^{(l)}_{ij} \sim  \mathcal{N}\left(0, \sigma_{w,l} \right)$.

As a general remark, note that the scaling factor of the output $\prod^L_{m=1}\sigma_{w,m}$ quickly approaches zero for increasing depth, which could render one-shot pruning numerically infeasible. 
For that reason, we propose a computationally cheap rescaling procedure that allows us to find significantly sparser strong lottery tickets.
This is also helpful for maintaining the trainability of the pruned network \citep{rescaleInit}, which we discuss next in the context of initializations with nonzero biases.

\subsection{Signal Preservation}
The question remains whether input and gradient signals can still propagate through the large original network with such an initialization.
A common criterion to prevent initial vanishing or exploding gradients, in particular in mean field analyses \citet{meanfield}, is to ensure that the squared signal norm of the input can propagate through the initial network. 
To bound the second moment of the squared output, we generalize Cor.~3 for normal distributions in \citet{dyniso} to symmetric weight and bias distributions.
\begin{lemma}\label{thm:train}
Assume that the weights and biases of a fully-connected deep neural network $f$ are drawn independently from distributions that are symmetric around the origin $0$ with variances $\sigma^{2}_{w,l}$ or $\sigma^{2}_{b,l}$, respectively. Then, for every input $\bm{x_0}$, the second moment of the output is
\begin{equation}\label{eq:Enorm} 
\begin{split}
 \mathbb{E}\left(||\bm{f(x_0)}||^2_2 \right) = & ||\mathbf{ x}^{(0)}||^2_2  \left(\prod^L_{l=1}\frac{n_l \sigma^{2}_{w,l}}{2}\right)  + \sigma^{2}_{b,L} \frac{n_L}{2} 
 +   \left(\sum^{L-1}_{l=1} \sigma^{2}_{b,l} \frac{n_l}{2} \left(\prod^{L}_{k=l+1}\frac{n_k\sigma^{2}_{w,k}}{2}\right)\right).
\end{split}
\end{equation}
\end{lemma}
For $\sigma^{2}_{w,l} \approx 2/n_l$ (as usually realized by He initialization \cite{HeInit}) and our choice $\sigma_{b,l} = \prod^l_{m=1}\sigma_{w,m}$, this implies that $\mathbb{E}\left(||\bm{f(x_0)}||^2_2 \right) \approx ||\mathbf{ x}^{(0)}||^2_2 + 1$, which prevents initial signal and gradient explosions even for high depth $L$.

\paragraph{'Looks linear` and orthogonal initialization}
The above lemma assumes that the weights and biases are drawn independently at random. 
This does not hold for orthogonal weight initialization, whose benefits have been highlighted in numerous works in general \citep{dynIsometry,spectralUniversality,orthonormalInit} and in particular for lottery ticket pruning \citep{orthoRepair}.
The marginal distribution of each weight entry is still normally distributed as $w^{(l)}_{ij} \sim \mathcal{N}\left(0, 1/n_l\right)$ so that our lottery ticket existence proof (in Section \ref{sec:theorems}) still applies approximately to this setting.
However, the main advantage of orthogonal weight initialization for trainability is usually induced by (approximate) dynamical isometry.
For ReLU activation functions, this is not achievable simply by initializing the whole matrix $W^{(l)}$ as orthogonal \citep{dynIsometry,spectralUniversality,dyniso}.
The solution is in fact based on the same insight that enables all current lottery ticket existence proofs for ReLUs, where the identity can be represented by $x = \phi(x) - \phi(-x)$.

As dynamical isometry can be achieved by a Jacobian that is similar to the identity, \citet{dyniso,orthoCNN} could ensure perfect dynamical isometry for ReLUs by a 'looks linear` initialization of the weight matrix and zero biases so that the full signal is always preserved at initialization.
Effectively, each neural network layer computes $\tilde{\bm{x}}^{(l)} = \phi\left(\bm{W}^{(l)}_0 \tilde{\bm{x}}^{(l-1)}\right) - \phi\left(-\bm{W}^{(l)}_0 \tilde{\bm{x}}^{(l-1)}\right)$, where the matrix $\bm{W}^{(l)}_0 \in \mathbb{R}^{n_l/2} \times \mathbb{R}^{n_{l-1}/2}$ is orthogonal. 
Extending this idea by nonzero biases corresponds, effectively, to $\tilde{\bm{x}}^{(l)} = \phi\left(\bm{W}^{(l)}_0 \tilde{\bm{x}}^{(l-1)} + \bm{b}_l\right) - \phi\left(-\bm{W}^{(l)}_0 \tilde{\bm{x}}^{(l-1)}- \bm{b}_l\right)$, where the matrix $\bm{W}^{(l)}_0 \in \mathbb{R}^{n_l/2} \times \mathbb{R}^{n_{l-1}/2}$.
Concretely, we define 
\begin{align}
 & \bm{W}^{(l)}  =   \left[ {\begin{array}{cc}
  \bm{W}^{(l)}_{0}  & -\bm{W}^{(l)}_{0} \\
   -\bm{W}^{(l)}_{0} & \bm{W}^{(l)}_{0} \\
  \end{array} } \right],\qquad
   \bm{b}^{(l)} =  \left[ {\begin{array}{cc} \bm{b}^{(l)}_{0}  & -\bm{b}^{(l)}_{0} \end{array} } \right]
\end{align}
for orthogonal, nonzero bias initialization with $\bm{b}^{(l)}_{0} \sim \mathcal{N}\left(0, \sigma^2_{b,l} I\right)$ independently from the weights.
Note that each entry of the weight matrix is again distributed as $w^{(l)}_{ij} \sim \mathcal{N}\left(0, 2/n_l\right)$ as in case of He initialization.
How should we choose the variance $\sigma^2_{b,l}$ of the biases?
Similarly to Lemma~\ref{thm:train}, we can derive the variance of the output signal as
\begin{align}\label{eq:orthovar} 
\begin{split}
 \mathbb{E}\left(||\bm{f(x_0)}||^2_2 \right) = & ||\mathbf{ x}^{(0)}||^2_2  +   \sum^{L}_{l=1} \sigma^{2}_{b,l}  \frac{n_l}{2}.
\end{split}
\end{align}
Initially, the additional $\sum^{L}_{l=1} \sigma^{2}_{b,l}  \frac{n_l}{2}$ is easier to control than in Lemma~\ref{thm:train}, exactly because the weights do not scale the biases randomly.
Note that we could improve this further and initialize the bias also dependent on the weights and make them cancel out to achieve again perfect dynamical isometry at initialization and $\mathbb{E}\left(||\bm{f(x_0)}||^2_2 \right) = ||\mathbf{ x}^{(0)}||^2_2 $.
However, this would depend on a carefully chosen dependence between weights and biases that gets destroyed during training and/or pruning.
For that reason, we still assume a situation similar to Lemma~\ref{thm:train} and initialize $\sigma_{b,l} = \prod^l_{m=1}\sigma_{w,m}$.
This case is also supported (at least approximately) by our lottery ticket existence proof and respects the scaling of parameters as outlined in Lemma~\ref{thm:scaling}.
With the trainability of randomly initialized networks, we have fulfilled the first criterion of a good initialization proposal for nonzero biases. 
The second criterion, the existence of lottery tickets, is discussed in the next section.

\section{Existence of Lottery Tickets with Nonzero Biases}\label{sec:theorems}

\begin{figure*}
    \centering
    \begin{subfigure}[t]{0.17\textwidth}
    \centering
    \includegraphics[width=\textwidth]{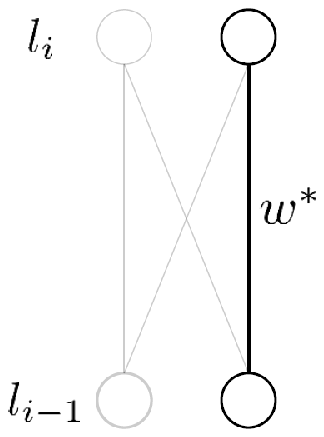}
    \caption{Target network $f$}
    \end{subfigure}
    \hfill
    \begin{subfigure}[t]{0.43\textwidth}
    \centering
    \includegraphics[width=\textwidth]{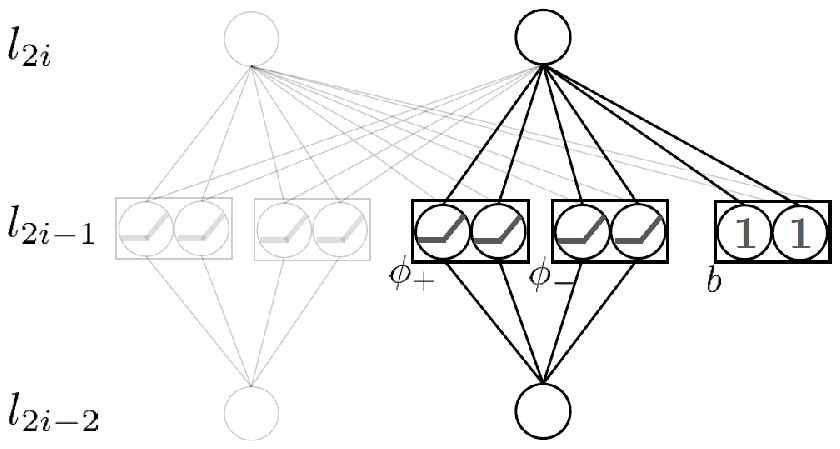}
    \caption{Mother network $f_0$}
    \end{subfigure}
    \hfill
    \begin{subfigure}[t]{0.38\textwidth}
    \centering
    \includegraphics[width=\textwidth]{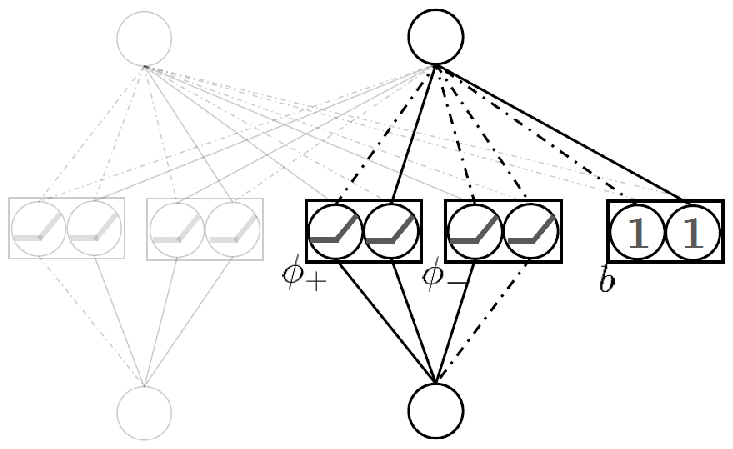}
    \caption{Strong lottery ticket $f_{\epsilon}$}
    \end{subfigure}
    \caption{\textit{Existence proof network construction.} \textbf{(a)} Two subsequent layers $l_i, l_{i-1}$ of the target network $f$ of depth $L$. \textbf{(b)} A mother network $f_0$ of depth $2L$, where intermediate layers with degree-one nodes are squeezed in between any pair of subsequent layers. The size of this layer is determined by the reduction on the subset sum problem. Bold neurons and connections are used to approximate $w^*$ of the target $f$ and the bias of the neurons in $l_i$ of $f$. \textbf{(c)} Lottery ticket $f_{\epsilon}$ obtained from the mother network by pruning connections (dashed lines).}
    \label{fig:proof_construction}
\end{figure*}

Proofs of the existence of lottery tickets have derived sufficient conditions under which pruning algorithms should have a good chance to find a winning ticket.

The first proof of the strong lottery ticket hypothesis \citep{malach2020proving} has shown that a weight-bounded deep neural target network of depth $L$ and width $n$ with ReLU activation functions is contained up to error $\epsilon$ with high probability in a larger deep neural network of double the depth of the target network, i.e., $2L$.
Their strong requirement on the width of the large network to be of polynomial order $O(n^5 L^2/\epsilon^2)$ or, under additional sparsity assumptions, $O(n^2 L^2/\epsilon^2)$, was subsequently improved to a logarithmic dependency of the form $O(n^2 \log(nL)/\epsilon)$ for weights that follow an unusual hyperbolic distribution \citep{orseau2020logarithmic} and $O(n \log(n L/\epsilon))$ for uniformly distributed weights \citep{pensia2020optimal}.
The improvement is achieved by the insight that many different parametrizations exist that can compute almost the same function as the target network.
All proofs have in common that two neural network layers are needed to approximate a neuron $\phi(\bm{w^T x})$, which explains the $2L$ depth requirement.

\subsection{SLTH with Nonzero Bias Initializations}

So far, existing works assume that the target network has zero biases.
This limits significantly the class of functions that we can hope to learn by pruning alone.
To extend the existence proofs, the first question that we have to answer is: 
How does the error propagate through a network with nonzero biases? 
Similar to Lem.~1 in \citet{orseau2020logarithmic}, we can deduce from the answer of how close each parameter $\theta_{\epsilon}$ needs to be to the target one how to guarantee an $\epsilon$ approximation of the entire network.
\begin{lemma}[Approximation propagation]\label{thm:approx}
Assume $\epsilon > 0$ and let the target network $f$ and its approximation $f_{\epsilon}$ have the same architecture.
If every parameter $\theta$ of $f$ and corresponding $\theta_{\epsilon}$ of $f_{\epsilon}$ in layer $l$ fulfils $|\theta_{\epsilon} - \theta| \leq \epsilon_{l}$ for 

\begin{multline}\label{eq:epsTheta}
\epsilon_l := \epsilon \left(L \sqrt{n_l k_{l,\text{max}}} \left(1+  \sup_{x \in {[-1,1]}^{n_0}} \norm{\bm{x}^{(l)}}_{1}\right)\right)^{-1} 
\left( \prod^L_{k=l+1} \left(\norm{\bm{W}^{(l)}}_{\infty} + \epsilon/L \right) \right)^{-1},
\end{multline}
then it follows that $||f-f_{\epsilon}||_{\infty} \leq \epsilon$. 
\end{lemma}
The proof is provided in the supplement. 
Note that large weights in every layer could imply that $\epsilon_l$ is exponential in $L$.
However, if we assume bounded weights so that $||\bm{W}^{(l)}||_{\infty} \leq 1$, we receive a moderate scaling of $\epsilon_l = C \epsilon/L$, where $C$ depends on the maximum degree of the neurons $k_{l,\text{max}} \leq n_{l-1}+1$ and the size of the biases via $\sup_{x \in {[-1,1]}^{n_0}} ||\bm{x}^{(l)}||_{1}$.
As we expect each output component of the target network $f$ to be in $[0,1]$, reasonable choices of biases lead usually to $\sup_{x \in {[-1,1]}^{n_0}} ||\bm{x}^{(l)}||_{1} \leq n_{l-1}$ and thus $\epsilon_l = \epsilon/(L (n_l+1)(n_{l-1}+1)e)$. 
Otherwise, we could rescale all parameters and thus the output to ensure desirable scaling. 
However, this would come at the expense of adapting the allowed error $\epsilon$ accordingly.

Next, we extend the proof of the existence of lottery tickets in \citet{pensia2020optimal} to nonzero biases. 
In addition, we generalize it to domains $[-1,1]^{n_0}$ (instead of balls with radius $1$) and present sharper width estimates based on the in-degrees of neurons instead of the full target network width $n_l$. 
The big advantage of our initialization scheme is that we can directly transfer an approach that would assume uniformly distributed parameters in $\theta_i \sim U{([-1,1])}$.

\begin{theorem}[\textbf{Existence of lottery ticket}]\label{thm:LTexist2}
Assume that ${\epsilon, \delta \in {(0,1)}}$ and a target network $f$ with depth $L$ and architecture $\bar{n}$ are given. 
Each weight and bias of a larger deep neural network $f_{0}$ with depth $2L$ and architecture $\bar{n}_0$ is initialized independently, uniformly at random according to $w^{(l)}_{ij} \sim U{([-\sigma_{w,l}, \sigma_{w,l}])}$ and $b^{(l)}_{i} \sim U{([-\prod^l_{k=1}\sigma_{w,k}, \prod^l_{k=1}\sigma_{w,k}])}$. Then, with probability at least $1-\delta$, $f_{0}$ contains an approximation $f_{\epsilon} \subset f_{0}$ so that $||f- \lambda f_{\epsilon}||_{\infty} \leq \epsilon$ if for l = 1, ..., L
\begin{align}
n_{2l-1,0} = C n_{l-1} \log\left(\frac{k_{l-1,\text{max}} n_{l}}{\min\left\{\epsilon_l, \delta/L\right\}}\right)
  \text{ and } n_{2l,0} = n_l,
\end{align}
where $\epsilon_l$ is defined in Eq.~(\ref{eq:epsTheta}) and the output is scaled by $\lambda = \prod^{2L}_{l=1}\sigma^{-1}_{w,l}$.
\end{theorem}
A similar statement holds also for normal distributions, i.e., $w^{(l)}_{ij} \sim \mathcal{N}\left(0, \sigma^2_{w,l}\right)$ and $b^{(l)}_{i} \sim \mathcal{N}\left(0, \prod^l_{k=1}\sigma^2_{w,k}\right)$.
Note that, essentially, we receive the same scaling as in case of zero biases.
Only the maximum degree $k_{l-1,\text{max}}$ is modified by $+1$.
The reason is that we can treat each bias as an additional weight in the construction of a ticket.
We provide the full proof in the appendix and restrict ourselves in the following to the description of the main idea.\\
\emph{Proof Idea}.
Layer $2l-1$ and $2l$ of the large network serve the representation of the neurons in Layer $l$ of the target network, e.g, neuron $\phi\left(\sum_j w_{ij} x^{(l-1)}_j + b_i\right)$. 
By using the identity $x = \phi(x) - \phi(-x)$ for ReLUs, we can express the preactivation also as 
$\sum_j \phi\left(w_{ij} x^{(l-1)}_j\right) - \phi\left(-w_{ij} x^{(l-1)}_j\right) + \text{sign}(b_i) \phi(|b_i|)$.
We provide a visualization of this construction in Fig.~\ref{fig:proof_construction}, where we highlight the construction for a single weight and neuron in the target, through one term $\phi\left(w_{ij} x^{(l-1)}_j\right)$ and bias $\phi(|b_i^{(l-1)}|)$, which are each approximated through a subset sum construction of randomly initialized weights and biases in $f_0$.
The width of the intermediate layers needs to be only of order $\log(1/\epsilon)$ according to results by \citet{subsetsum} on the subset sum problem, which can be applied to finding lottery tickets \citep{pensia2020optimal}.
Accordingly, with probability at least $1-\delta$, for each parameter $\theta$ (i.e., $w_{ij}$, $-w_{ij}$, or $|b_i|$) exists a subset $S$ of $n$ uniformly distributed parameters $X_i\sim U([-1,1])$ so that $|\theta - \sum_{i\in S} X_i| < \epsilon_l$ for $n \geq C \log 1/\min\{\epsilon_l, \delta\}$. 
Repeating this argument in combination with union bounds over $k_i$ parameters per neuron, all neurons in a layer, and all layers, leads to the desired results.
While this construction is purely theoretical in nature, it turns out that the mother network and approximation of the target have properties that are comparable to real world lottery ticket settings.

\subsection{Proof Construction -- From Theory to Practice}

We, here, conduct a study on the efficacy of the existence proof construction with respect to the networks $f_0$ and $f_{\epsilon}$, the code is available online.\!\footnote{\oururl}
In particular, we show that the proof construction of mother networks $f_{0}$ of Thm.~\ref{thm:LTexist2} admits for efficient approximations $f_{\epsilon}$ for a given target network $f$. 
As a case study, we consider a LeNet~\cite{lecun-gradientbased-learning-applied-1998} architecture of the form  $784 \rightarrow 300 \rightarrow 100 \rightarrow 10$, where each fully connected layer is followed by ReLU activations. 
We train such a He initialized LeNet by Iterative Magnitude Pruning~\cite{frankle2019lottery} on MNIST, which results in a small target network with around 20k parameters, and an accuracy of $97.96$. 
We consider the resulting ticket as target network $f$.

As detailed in the Theorem \ref{thm:LTexist2}, each layer in the target network can be approximated by two subsequent layers in a mother network $f_0$ of depth $2L$. 
Our approximation to the target lottery ticket is a subnetwork of this mother network, $f_{\eps} \subset f_0$. 
We construct each layer of $f_{\eps}$ by solving the subset sum problem neuron-wise, aiming to find the smallest subset of weights in $f_0$ that allow to approximate each weight in $f$, as implied by the proof.
Our experiments verify that the constructed SLTs indeed do as well as the target network on MNIST data for our proposed nonzero initialization schemes, including a 'looks linear' orthogonal initialization (see App. Tab.~\ref{table:const_results}).
Furthermore, these tickets are of great sparsity with respect to the constructed mother network and only a factor of $\sim6$ larger than the (real) target, providing evidence that extremely sparse solutions exist in large, overparametrized neural networks.
Finally, although purely theoretical, the construction used within the proof leads to a mother network $f_0$ of architecture $784 \rightarrow 31400 \rightarrow 300 \rightarrow 12040 \rightarrow 100 \rightarrow 4040 \rightarrow 10$, which is well within the size of modern neural networks that are considered for pruning.

\section{Parameter Scaling during Pruning}

According to Sec.~\ref{sec:outscaling} and our existence proof, we expect that the output of a pruned network usually does not match the right target range.
The lottery ticket $f_{\epsilon}$ needs to be scaled by a factor $\lambda > 0$ (and usually $\lambda > 1$) .
In the existence proof, $\lambda = \prod^L_{m=1}\sigma_{w,m}$ but this factor can be vanishing small for very deep networks.
Furthermore, in many applications we do not know the exact size the of the network parameters and the output might also not be restricted to $[-1,1]$.
For these reasons, we propose to learn an appropriate output scaling factor $\lambda > 0$ that successively adapts the lottery ticket $f_{\epsilon}$ after each pruning epoch.
For regression minimizing the mean squared error with respect to $N$ data samples with targets $y_{i,s}$, this scalar can be easily computed as 
\begin{align}\label{eq:msescaling}
    \lambda_{\text{mse}} = \left(\sum^{N}_{s=1} \sum^{n_{L}}_{i=1} y_{i,s}x^{(L)}_{i,s}\right)/ \left(\sum^{N}_{s=1} \sum^{n_{L}}_{i=1} x^{(L)}_{i,s} x^{(L)}_{i,s}\right).
\end{align}
In case of a different loss $\mathcal{L}$, we only have to solve a one-dimensional optimization problem 
\begin{align}
 \min_{\lambda > 0} \mathcal{L}(y, \lambda x^{(L)})    
\end{align}
which can, for instance, be achieved with SGD. 
To distribute the scaling factor on the different layers, we use again the parameter transformation in Lemma~\ref{thm:scaling}, which is $w^{(l)}_{ij} = w^{(l)}_{ij} \lambda^{1/L}, \qquad  b^{(l)}_{i} = b^{(l)}_{i} \lambda^{l/L},$
which ensures that the overall output of the neural network is scaled by $\lambda$.

Combining this parameter rescaling for tickets found in each epoch with the SLT algorithm \texttt{edge-popup}, along with a slow annealing of the target sparsity throughout the course of pruning, we obtain \texttt{edge-popup-scaled}, for which we report pseudocode in Alg.~\ref{alg:edgepopup_scaled}. 
As we show next, this parameter rescaling allows us to obtain much sparser lottery tickets than pure pruning.

\subsection{A synthetic case study}

\begin{figure*}
    \centering
    \includegraphics[width=\linewidth]{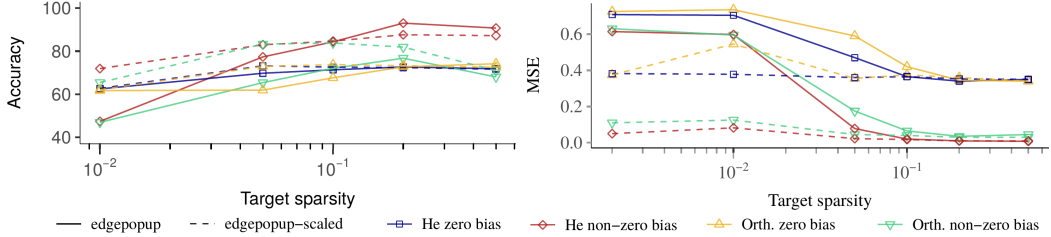}
    \caption{\textit{Strong lottery ticket pruning with \texttt{edge-popup}.} Shown are results for discovered tickets of different target sparsities, methods, and initialization schemes. Left: Ellipse data, higher is better (Classification). Right: Shifted ReLU data, lower is better (Regression).}
    \label{fig:synth_results}
\end{figure*}

We compare \texttt{edge-popup} with annealing of target sparsities against a version with scaling \texttt{edge-popup-scaled} for networks initialized with and without biases.
On two synthetic benchmarks, which turn out to be challenging settings for SLT pruning, we evaluate these methods for both He as well as `looks-linear' orthogonal initialization with and without nonzero bias extensions.
All experiments are carried out on commodity hardware.
\texttt{Edge-popup} training was conducted in the default way suggested by the original authors by SGD with momentum of $0.9$ and weight decay $0.0005$, combined with cosine annealing of the learning rate starting from $0.1$.
We train for $e=10$ epochs and $e_a=5$ annealing epochs in case of the shifted ReLU example and $e_a = 20$ annealing epochs in case of the ellipse and sparsity levels ${0.01, 0.05}$.
In case of higher sparsity levels, $e_a = 10$ annealing epochs are sufficient. 
During annealing, we slowly reduce the sparsity over time by $\rho^{i / e_a}$, where $\rho$ is the desired network sparsity, and $i$ is the current epoch.
We used a batch size of $32$ in all experiments and report mean  based on $5$ repetitions.
We make all code and data publicly available.\!\footnote{\oururl}

\subsubsection*{A Shifted ReLU}

First, we consider data of a regression problem that follows a shifted ReLU function $\phi_b(x)=\max(0, x+b)$ with $b=0.5$. 
We thus draw $N=10^4$ iid samples $x_i \sim U[-1,1]$ with targets $y_i = \phi(x_i+0.5) + n_i$ with independent Gaussian noise $n_i \sim \mathcal{N}(0, 0.01^2)$.
For a network of depth $5$, where each layer is of width $100$, we retrieve strong lottery tickets at target sparsities $\{0.002, 0.01, 0.05, 0.1, 0.2\}$. 
We report the mean squared error (MSE) of the strong tickets on the test set as mean across $5$ repetitions in Fig.~\ref{fig:synth_results} (right).

We observe that, consistently, the nonzero bias initialized networks enable the recovery of strong tickets with orders of magnitude smaller errors than in networks with zero-initialized bias.
Furthermore, \texttt{edge-popup-scaled} with proper rescaling of parameters consistently outperforms the vanilla (unscaled) algorithm for extreme sparsities.
At lower sparsity levels, rescaled \texttt{edge-popup} allows to retrieve well performing tickets that match the low error of their more dense counterparts, whereas vanilla \texttt{edge-popup} fails to find good tickets.

\subsubsection*{The Onion Slice}

Next, we consider a classification problem, where points are arranged in elliptic rings, and each point is labeled by the ring it appears in.
$N=10^4$ inputs are again sampled iid from uniform distributions $x_1, x_2 \sim U[-1,1]$ and one of four labels is assigned as target based on the value $y = 0.5 (x_1-0.3)^2 + 1.2 (x_2 + 0.5)^2$. 
Class boundaries are defined as $(0.2, 0.5, 0.7)$, while noise is introduced by flipping a label to a neighboring class with probability $0.01$.

For networks of depth $5$ and width $100$, we retrieve strong lottery tickets at target sparsities $\{0.01, 0.05, 0.1, 0.2, 0.5\}$.
We report the accuracy of tickets on the test set as mean across $5$ repetitions in Fig. \ref{fig:synth_results} (left).
Tickets pruned from networks initialized with nonzero biases outperform their zero-bias counterparts. 
An exception to this rule is given by the unscaled \texttt{edge-popup} for sparsity $0.01$, where both initialization approaches (with nonzero and zero biases) show unsatisfactory performance.
In contrast, the rescaled \texttt{edge-popup} with nonzero bias He initialization is still able to retrieve extremely sparse tickets with more than $10$ accuracy points margin to all other approaches. 

\section{Discussion}\label{sec:discussion}

Strong lottery tickets, as currently defined in the literature, do not lend themselves as universal function approximators due to the limitation to the standard zero bias initialization schemes for neural networks.
We hence transferred the strong lottery ticket hypothesis to neural networks with potentially nonzero initial biases and proved the existence of strong lottery tickets under realistic conditions with respect to the network width and initialization scheme.
This generalization equips training by pruning for strong lottery tickets with the universal approximation property.

Along with the proof, we have extended standard initialization schemes to nonzero biases and formally shown that our proposal defines well trainable neural networks, while they support the existence of strong lottery tickets.
These initializations include the 'looks-linear' approach \citep{dyniso,orthoCNN} that ensures initial dynamical isometry of ReLU networks, which often leads to favorable training properties.

Based on our theoretical insights, we have derived a parameter scaling strategy that enables pruning algorithms to find sparser strong lottery tickets.
We have extended the \texttt{edge-popup} algorithm \citep{ramanujan2019whats} for strong lottery ticket pruning accordingly and demonstrated the utility of our innovations on two case studies.
For imaging data, our nonzero bias initializations are well trainable, but the current generation of algorithms lacks the ability to draw an advantage over zero bias initialized networks (see App.~\ref{app:cifar}), likely due to their parameter-inefficacy that has been reported before~\citep{fischer2022planting}.
With the development of pruning algorithms that can find highly sparse strong lottery tickets, we anticipate that nonzero bias initializations are important for lottery ticket pruning in theory as well as practice. 


\bibliography{lottery}

\begin{thebibliography}{44}
\providecommand{\natexlab}[1]{#1}
\providecommand{\url}[1]{\texttt{#1}}
\expandafter\ifx\csname urlstyle\endcsname\relax
  \providecommand{\doi}[1]{doi: #1}\else
  \providecommand{\doi}{doi: \begingroup \urlstyle{rm}\Url}\fi

\bibitem[Balduzzi et~al.(2017)Balduzzi, Frean, Leary, Lewis, Ma, and
  McWilliams]{orthoCNN}
D.~Balduzzi, M.~Frean, L.~Leary, J.~P. Lewis, K.~W.-D. Ma, and B.~McWilliams.
\newblock The shattered gradients problem: If resnets are the answer, then what
  is the question?
\newblock In \emph{International Conference on Machine Learning}, 2017.

\bibitem[Burkholz and Dubatovka(2019)]{dyniso}
R.~Burkholz and A.~Dubatovka.
\newblock Initialization of {ReLUs} for dynamical isometry.
\newblock In \emph{Advances in Neural Information Processing Systems},
  volume~32, 2019.

\bibitem[Chen et~al.(2021)Chen, Cheng, Wang, Gan, Liu, and Wang]{elasticLTH}
X.~Chen, Y.~Cheng, S.~Wang, Z.~Gan, J.~Liu, and Z.~Wang.
\newblock The elastic lottery ticket hypothesis.
\newblock In \emph{Advances in Neural Information Processing Systems}, 2021.

\bibitem[Dong et~al.(2017)Dong, Chen, and Pan]{dong2017surgeon}
X.~Dong, S.~Chen, and S.~J. Pan.
\newblock Learning to prune deep neural networks via layer-wise optimal brain
  surgeon.
\newblock In \emph{Advances in Neural Information Processing Systems}, 2017.

\bibitem[Fischer and Burkholz(2022)]{fischer2022planting}
J.~Fischer and R.~Burkholz.
\newblock Plant 'n' seek: Can you find the winning ticket?
\newblock In \emph{International Conference on Learning Representations}, 2022.

\bibitem[Frankle and Carbin(2019)]{frankle2019lottery}
J.~Frankle and M.~Carbin.
\newblock The lottery ticket hypothesis: Finding sparse, trainable neural
  networks.
\newblock In \emph{International Conference on Learning Representations}, 2019.

\bibitem[Frankle et~al.(2020)Frankle, Dziugaite, Roy, and Carbin]{rewind}
J.~Frankle, G.~K. Dziugaite, D.~Roy, and M.~Carbin.
\newblock Linear mode connectivity and the lottery ticket hypothesis.
\newblock In \emph{International Conference on Machine Learning}, 2020.

\bibitem[Glorot and Bengio(2010)]{GlorotInit}
X.~Glorot and Y.~Bengio.
\newblock Understanding the difficulty of training deep feedforward neural
  networks.
\newblock In \emph{International Conference on Artificial Intelligence and
  Statistics}, volume~9, pages 249--256, May 2010.

\bibitem[Hanin and Rolnick(2019)]{complLinReg}
B.~Hanin and D.~Rolnick.
\newblock Deep relu networks have surprisingly few activation patterns.
\newblock In \emph{Advances in Neural Information Processing Systems}, 2019.

\bibitem[Hassibi and Stork(1992)]{hassibi1992second}
B.~Hassibi and D.~G. Stork.
\newblock Second order derivatives for network pruning: Optimal brain surgeon.
\newblock In \emph{International Conference on Neural Information Processing
  Systems}, 1992.

\bibitem[Hayou et~al.(2021)Hayou, Ton, Doucet, and Teh]{rescaleInit}
S.~Hayou, J.-F. Ton, A.~Doucet, and Y.~W. Teh.
\newblock Robust pruning at initialization.
\newblock In \emph{International Conference on Learning Representations}, 2021.

\bibitem[He et~al.(2015)He, Zhang, Ren, and Sun]{HeInit}
K.~He, X.~Zhang, S.~Ren, and J.~Sun.
\newblock Delving deep into rectifiers: Surpassing human-level performance on
  imagenet classification.
\newblock In \emph{Proceedings of the IEEE international conference on computer
  vision}, pages 1026--1034, 2015.

\bibitem[LeCun et~al.(1990)LeCun, Denker, and Solla]{lecun1990optimal}
Y.~LeCun, J.~S. Denker, and S.~A. Solla.
\newblock Optimal brain damage.
\newblock In \emph{Advances in neural information processing systems}, pages
  598--605, 1990.

\bibitem[LeCun et~al.(1998)LeCun, Bottou, Bengio, and
  Haffner]{lecun-gradientbased-learning-applied-1998}
Y.~LeCun, L.~Bottou, Y.~Bengio, and P.~Haffner.
\newblock Gradient-based learning applied to document recognition.
\newblock In \emph{Proceedings of the IEEE}, volume~86, pages 2278--2324, 1998.
\newblock URL
  \url{http://citeseerx.ist.psu.edu/viewdoc/summary?doi=10.1.1.42.7665}.

\bibitem[Lee et~al.(2019)Lee, Ajanthan, and Torr]{snip}
N.~Lee, T.~Ajanthan, and P.~H.~S. Torr.
\newblock Snip: single-shot network pruning based on connection sensitivity.
\newblock In \emph{International Conference on Learning Representations}, 2019.

\bibitem[Lee et~al.(2020)Lee, Ajanthan, Gould, and Torr]{orthoRepair}
N.~Lee, T.~Ajanthan, S.~Gould, and P.~H.~S. Torr.
\newblock A signal propagation perspective for pruning neural networks at
  initialization.
\newblock In \emph{International Conference on Learning Representations}, 2020.

\bibitem[Li et~al.(2017)Li, Kadav, Durdanovic, Samet, and
  Graf]{li2017pruneconv}
H.~Li, A.~Kadav, I.~Durdanovic, H.~Samet, and H.~P. Graf.
\newblock Pruning filters for efficient convnets.
\newblock In \emph{International Conference on Learning Representations}, 2017.

\bibitem[Liu et~al.()Liu, Yuan, Che, Shen, Ma, Jin, Ren, Tang, Liu, and
  Wang]{liu2021:finetune}
N.~Liu, G.~Yuan, Z.~Che, X.~Shen, X.~Ma, Q.~Jin, J.~Ren, J.~Tang, S.~Liu, and
  Y.~Wang.
\newblock Lottery ticket preserves weight correlation: Is it desirable or not?
\newblock In \emph{International Conference on Machine Learning}.

\bibitem[Liu et~al.(2021)Liu, Yuan, Che, Shen, Ma, Jin, Ren, Tang, Liu, and
  Wang]{weightcor}
N.~Liu, G.~Yuan, Z.~Che, X.~Shen, X.~Ma, Q.~Jin, J.~Ren, J.~Tang, S.~Liu, and
  Y.~Wang.
\newblock Lottery ticket preserves weight correlation: Is it desirable or not?
\newblock In \emph{International Conference on Machine Learning}, 2021.

\bibitem[Lueker(1998)]{subsetsum}
G.~S. Lueker.
\newblock Exponentially small bounds on the expected optimum of the partition
  and subset sum problems.
\newblock \emph{Random Structures \& Algorithms}, 12\penalty0 (1):\penalty0
  51--62, 1998.

\bibitem[Ma et~al.(2021)Ma, Yuan, Shen, Chen, Chen, Chen, Liu, Qin, Liu, Wang,
  and Wang]{sanity2}
X.~Ma, G.~Yuan, X.~Shen, T.~Chen, X.~Chen, X.~Chen, N.~Liu, M.~Qin, S.~Liu,
  Z.~Wang, and Y.~Wang.
\newblock Sanity checks for lottery tickets: Does your winning ticket really
  win the jackpot?
\newblock In \emph{Advances in Neural Information Processing Systems}, 2021.

\bibitem[Malach et~al.(2020)Malach, Yehudai, Shalev-Schwartz, and
  Shamir]{malach2020proving}
E.~Malach, G.~Yehudai, S.~Shalev-Schwartz, and O.~Shamir.
\newblock Proving the lottery ticket hypothesis: Pruning is all you need.
\newblock In \emph{International Conference on Machine Learning}, 2020.

\bibitem[Molchanov et~al.(2017)Molchanov, Tyree, Karras, Aila, and
  Kautz]{molchanov2017pruneinf}
P.~Molchanov, S.~Tyree, T.~Karras, T.~Aila, and J.~Kautz.
\newblock Pruning convolutional neural networks for resource efficient
  inference.
\newblock In \emph{International Conference on Learning Representations}, 2017.

\bibitem[Orseau et~al.(2020)Orseau, Hutter, and
  Rivasplata]{orseau2020logarithmic}
L.~Orseau, M.~Hutter, and O.~Rivasplata.
\newblock Logarithmic pruning is all you need.
\newblock \emph{Advances in Neural Information Processing Systems}, 33, 2020.

\bibitem[Pennington et~al.(2017)Pennington, Schoenholz, and
  Ganguli]{dynIsometry}
J.~Pennington, S.~S. Schoenholz, and S.~Ganguli.
\newblock Resurrecting the sigmoid in deep learning through dynamical isometry:
  theory and practice.
\newblock In \emph{Advances in Neural Information Processing Systems}, 2017.

\bibitem[Pennington et~al.(2018)Pennington, Schoenholz, and
  Ganguli]{spectralUniversality}
J.~Pennington, S.~S. Schoenholz, and S.~Ganguli.
\newblock The emergence of spectral universality in deep networks.
\newblock In \emph{International Conference on Artificial Intelligence and
  Statistics}, 2018.

\bibitem[Pensia et~al.(2020)Pensia, Rajput, Nagle, Vishwakarma, and
  Papailiopoulos]{pensia2020optimal}
A.~Pensia, S.~Rajput, A.~Nagle, H.~Vishwakarma, and D.~Papailiopoulos.
\newblock Optimal lottery tickets via subset sum: Logarithmic
  over-parameterization is sufficient.
\newblock In \emph{Advances in Neural Information Processing Systems},
  volume~33, pages 2599--2610, 2020.

\bibitem[Ramanujan et~al.(2020{\natexlab{a}})Ramanujan, Wortsman, Kembhavi,
  Farhadi, and Rastegari]{edgepopup}
V.~Ramanujan, M.~Wortsman, A.~Kembhavi, A.~Farhadi, and M.~Rastegari.
\newblock What's hidden in a randomly weighted neural network?
\newblock In \emph{Conference on Computer Vision and Pattern Recognition},
  2020{\natexlab{a}}.

\bibitem[Ramanujan et~al.(2020{\natexlab{b}})Ramanujan, Wortsman, Kembhavi,
  Farhadi, and Rastegari]{ramanujan2019whats}
V.~Ramanujan, M.~Wortsman, A.~Kembhavi, A.~Farhadi, and M.~Rastegari.
\newblock What's hidden in a randomly weighted neural network?
\newblock In \emph{Computer Vision and Pattern Recognition}, pages
  11893--11902, 2020{\natexlab{b}}.

\bibitem[Savarese et~al.(2020{\natexlab{a}})Savarese, Silva, and Maire]{LTreg}
P.~Savarese, H.~Silva, and M.~Maire.
\newblock Winning the lottery with continuous sparsification.
\newblock In \emph{Advances in Neural Information Processing Systems},
  2020{\natexlab{a}}.

\bibitem[Savarese et~al.(2020{\natexlab{b}})Savarese, Silva, and
  Maire]{sigmoidl0}
P.~Savarese, H.~Silva, and M.~Maire.
\newblock Winning the lottery with continuous sparsification.
\newblock In \emph{Advances in Neural Information Processing Systems},
  volume~33, pages 11380--11390, 2020{\natexlab{b}}.

\bibitem[Saxe et~al.(2014)Saxe, McClelland, and Ganguli]{orthonormalInit}
A.~M. Saxe, J.~L. McClelland, and S.~Ganguli.
\newblock Exact solutions to the nonlinear dynamics of learning in deep linear
  neural networks.
\newblock In \emph{International Conference on Learning Representations}, 2014.

\bibitem[Scarselli and Tsoi(1998)]{scarselli1998universal}
F.~Scarselli and A.~C. Tsoi.
\newblock Universal approximation using feedforward neural networks: A survey
  of some existing methods, and some new results.
\newblock \emph{Neural Netw.}, 11\penalty0 (1):\penalty0 15–37, Jan. 1998.

\bibitem[Schoenholz et~al.(2017)Schoenholz, Gilmer, Ganguli, and
  Sohl{-}Dickstein]{meanfield}
S.~S. Schoenholz, J.~Gilmer, S.~Ganguli, and J.~Sohl{-}Dickstein.
\newblock Deep information propagation.
\newblock In \emph{International Conference on Learning Representations}, 2017.

\bibitem[Srinivas and Babu(2016)]{srinivas2016gendropout}
S.~Srinivas and R.~V. Babu.
\newblock Generalized dropout.
\newblock \emph{CoRR}, abs/1611.06791, 2016.

\bibitem[Su et~al.(2020)Su, Chen, Cai, Wu, Gao, Wang, and Lee]{sanity}
J.~Su, Y.~Chen, T.~Cai, T.~Wu, R.~Gao, L.~Wang, and J.~D. Lee.
\newblock Sanity-checking pruning methods: Random tickets can win the jackpot.
\newblock In \emph{Advances in Neural Information Processing Systems}, 2020.

\bibitem[Tanaka et~al.(2020)Tanaka, Kunin, Yamins, and Ganguli]{synflow}
H.~Tanaka, D.~Kunin, D.~L. Yamins, and S.~Ganguli.
\newblock Pruning neural networks without any data by iteratively conserving
  synaptic flow.
\newblock In \emph{Advances in Neural Information Processing Systems}, 2020.

\bibitem[Verdenius et~al.(2020)Verdenius, Stol, and Forré]{snipit}
S.~Verdenius, M.~Stol, and P.~Forré.
\newblock Pruning via iterative ranking of sensitivity statistics, 2020.

\bibitem[Wang et~al.(2020)Wang, Zhang, and Grosse]{grasp}
C.~Wang, G.~Zhang, and R.~B. Grosse.
\newblock Picking winning tickets before training by preserving gradient flow.
\newblock In \emph{International Conference on Learning Representations}, 2020.

\bibitem[Weigend et~al.(1991)Weigend, Rumelhart, and Huberman]{weightelim}
A.~Weigend, D.~Rumelhart, and B.~Huberman.
\newblock Generalization by weight-elimination with application to forecasting.
\newblock In \emph{Advances in Neural Information Processing Systems}, 1991.

\bibitem[Yang et~al.(2019)Yang, Pennington, Rao, Sohl-Dickstein, and
  Schoenholz]{batchMFT}
G.~Yang, J.~Pennington, V.~Rao, J.~Sohl-Dickstein, and S.~S. Schoenholz.
\newblock A mean field theory of batch normalization.
\newblock In \emph{International Conference on Learning Representations}, 2019.

\bibitem[You et~al.(2020)You, Li, Xu, Fu, Wang, Chen, Baraniuk, Wang, and
  Lin]{earlybird}
H.~You, C.~Li, P.~Xu, Y.~Fu, Y.~Wang, X.~Chen, R.~G. Baraniuk, Z.~Wang, and
  Y.~Lin.
\newblock Drawing early-bird tickets: Toward more efficient training of deep
  networks.
\newblock In \emph{International Conference on Learning Representations}, 2020.

\bibitem[Zhang et~al.(2021)Zhang, Jin, Zhang, Zhou, Zhao, Ren, Liu, Wu, Jin,
  and Dou]{validateManifold}
Z.~Zhang, J.~Jin, Z.~Zhang, Y.~Zhou, X.~Zhao, J.~Ren, J.~Liu, L.~Wu, R.~Jin,
  and D.~Dou.
\newblock Validating the lottery ticket hypothesis with inertial manifold
  theory.
\newblock In \emph{Advances in Neural Information Processing Systems}, 2021.

\bibitem[Zhou et~al.(2019)Zhou, Lan, Liu, and Yosinski]{supermask}
H.~Zhou, J.~Lan, R.~Liu, and J.~Yosinski.
\newblock Deconstructing lottery tickets: Zeros, signs, and the supermask.
\newblock In \emph{Advances in Neural Information Processing Systems}, 2019.

\end{thebibliography}

\newpage

\appendix

\section{Theory}
In the following section, we present the proofs of the theorems and lemmas of the main manuscript.
\subsection{Motivation: Factorization of univariate zero-bias networks}
\begin{lemma*}
  Univariate neural networks with ReLU activations  $f:\reals\rightarrow\reals^{n_L}$ of arbitrary depth $L$ and layer widths $n_1,\ldots,n_L$, and without biases, represent a function $f(x) = \bm{W}_{+} \phi(x) + \bm{W}_{-} \phi(-x), \bm{W}_{+},\bm{W}_{-}\in\reals^{n_L \times 1}$.
\end{lemma*}
\begin{proof}
  We prove by induction over the number of hidden layers $L$ of the network.
  First, assume that $L=1$, i.e. $f^{(1)}(x)=\phi(\bm{W}^{(1)} x)$, with $\bm{W}^{(1)}\in\reals^{n_1 \times 1}$.
  For any output neuron $f_j^{(1)}(x),~j= 1,\ldots, n_1$, we have
  \begin{align}
    f_j^{(1)}(x) =& \sum_{i=1}^{n_0} \phi(\bm{W}^{(1)}_i x)\\
              =& \begin{cases}
                    \sum_{i=1}^{n_0} \bm{W}^{(1)}_i I(\bm{W}^{(1)}_i > 0) \phi(x) ,& \text{if } x\geq 0\\
                    \sum_{i=1}^{n_0} -\bm{W}^{(1)}_i I(\bm{W}^{(1)}_i < 0) \phi(-x),& \text{otherwise,}
                \end{cases}
  \end{align}
  where $I(.)$ is the indicator function.
  For networks with a single layer, we thus get $f^{(1)}(x) = \bm{W}^{(1)}_{+} \phi(x) + \bm{W}^{(1)}_{-} \phi(-x)$ with
  \begin{align}
    \bm{W}^{(1)}_{+} =& \bm{W}^{(1)}_i I(\bm{W}^{(1)}_i > 0) \,,\\
    \bm{W}^{(1)}_{-} =& -\bm{W}^{(1)}_i I(\bm{W}^{(1)}_i < 0)\,,
  \end{align}
  which proves our claim for $L=1$.
  
  Next, our induction hypothesis is that $f_j^{(l)} = \bm{W}^{(l)}_{+} \phi(x) + \bm{W}^{(l)}_{-}\phi(-x)$.
  For networks with $l+1$ layer, $f^{(l+1)}(x) = \bm{W}^{(l+1)} \phi\left(f^{(l)}(x)\right), \bm{W}^{(l+1)}\in\reals^{n_{l+1} \times n_l}$, for each
  output neuron $f_j^{(l+1)}(x),~j=1,\ldots,n_{l+1}$ we obtain
  \begin{align}
    f_j^{(l+1)}(x) =& \sum_{i=1}^{n_{l+1}}\phi\left(f^{(l)}(x)\right)_i \bm{W}^{(l+1)}_{ji}\\
                   =& \sum_{i=1}^{n_{l+1}}\phi\left(\bm{W}^{(l)}_{+} \phi(x) + \bm{W}^{(l)}_{-}\phi(-x)\right)_i \bm{W}^{(l+1)}_{ji} \ \ (\text{induction hypothesis})\\
                   =& \begin{cases}
                        \sum_{i=1}^{n_{l+1}}\phi\left(\bm{W}^{(l)}_{+} x\right)_i \bm{W}^{(l+1)}_{ji} ,& \text{if } x \geq 0\\
                        \sum_{i=1}^{n_{l+1}}\phi\left(-\bm{W}^{(l)}_{-}x\right)_i \bm{W}^{(l+1)}_{ji} ,& \text{otherwise}
                      \end{cases} \\
                   =& \begin{cases}
                        \sum_{i=1}^{n_{l+1}}\bm{W}^{(l)}_{+,i} \bm{W}^{(l+1)}_{ji} x ,~~~~~~~~~~~& \text{if } x \geq 0 \land \bm{W}_{+,i} > 0\\
                        \sum_{i=1}^{n_{l+1}}-\bm{W}^{(l)}_{-,i} \bm{W}^{(l+1)}_{ji} x ,&\text{if } x < 0 \land \bm{W}_{-,i} > 0 \\
                        0 ,&\text{otherwise.}
                      \end{cases}
  \end{align}
  Hence, we get 
  \begin{align}
    \bm{W}^{(l+1)}_{+}=& \bm{W}^{(l+1)}\bm{W}^{(l)}_{+}\\
    \bm{W}^{(l+1)}_{-}=&- \bm{W}^{(l+1)}\bm{W}^{(l)}_{-},
  \end{align}
  which concludes the induction step, and the proof.
\end{proof}

\subsection{ReLU networks without bias are not universal approximators}
\label{app:counterex_universal}
We, here, provide the necessary derivations for the counter-examples used in Theorem~\ref{thm:nonUniversalApprox}.
For the example $g(x)=0.5, x\in[-1,1]$, by minimizing the mean squared error (MSE), we get a loss of
\[
  \mathcal{L}(x; w_+, w_-) = \int_{-1}^{1}\left(g(x) - f(x)\right)^2 dx,
\]
where $f(x)$ is the neural network, which from Lemma \ref{lem:networkfact} we know is only dependent on $w_+,w_- \in \reals$.
Solving this integral, we obtain
\begin{align}
\mathcal{L}(x; w_+, w_-) =& \int_{-1}^{1}\left(g(x) - f(x)\right)^2 dx\\
                         =& \int_{-1}^0 (0.5 - w_-\phi(-x))^2 dx \\
                          &+ \int_0^1 (0.5 - w_+ \phi(x))^2 dx~~~~~~~~~~~~~~~~~~~~~~~~(\text{Lemma \ref{lem:networkfact}}) \\
                         =&\int_0^1 0.5 - w_-\phi(x) + w_-^2 \phi(x)^2 \\
                          &- w_+ \phi(x) + w_+^2\phi(x)^2 dx \\
                         =&\int_0^1 0.5 - w_-x + w_-^2 x^2 - w_+ x + w_+^2x^2 dx~~~~(\text{Def. } \phi(x) \text{ for } x>0) \\
                         =& \left[0.5x -0.5w_-x^2+\frac{1}{3}w_-^2x^3 \right. \\
                          &\left. -0.5w_+x^2+\frac{1}{3}w_+^2x^3+C\right]_0^1\,.~~~~~~~~~~~~~~~~~~~~~(\text{Primitive function})
\end{align}
We are interested in the network, and hence parameters $w_+^*, w_-^*$, that minimize the loss.
Thus, we solve $\frac{\mathcal{L}(x; w_+, w_-)}{dw_+}\stackrel{!}{=} 0$, which yields $w_+^*=\frac{3}{4}$,
and $\frac{\mathcal{L}(x; w_+, w_-)}{dw_-}\stackrel{!}{=} 0$, which yields $w_-^*=\frac{3}{4}$.
We can directly see from the shape of the function that these values are indeed a minimum, and not a maximum.
Plugging this back into the primitive function above, we obtain $\mathcal{L}(x; w_+^*, w_-^*) = \frac{1}{8}$.
Hence, for any $\epsilon < \frac{1}{8}$ there does not exist a network that can approximate the function to this error.

For a slightly more complicated function, without explicit offset, we consider $g(x)=e^x, x\in[-1,1]$.
Analogue to above, we first minimize the MSE
\begin{align}
\mathcal{L}(x; w_+, w_-) =& \int_{-1}^0 (e^{x} - w_-\phi(-x))^2 dx \\
                          &+ \int_0^1 (e^{x} - w_+ \phi(x))^2 dx \\
                         =&\int_0^1 e^{-2x} - 2e^{-x}w_-x + w_-^2 x^2 \\
                          &+ e^{2x}- 2e^x w_+ x + w_+^2x^2 dx \\
                         =& \left[-0.5e^{-2x} + 2e^{-x}w_- x + 2e^{-x}w_- + \frac{1}{3}w_-^2x^3 \right. \\
                          &\left. +0.5e^{2x} - 2e^xw_+x + 2e^xw_+ +\frac{1}{3}w_+^2x^3+C\right]_0^1\,.
\end{align}
For $\frac{\mathcal{L}(x; w_+, w_-)}{dw_+}\stackrel{!}{=} 0$, we get $w_+^*=3$,
and $\frac{\mathcal{L}(x; w_+, w_-)}{dw_-}\stackrel{!}{=} 0$, which gives $w_-^*=-3(2e^{-1}-1)$.
Again, we can observe from the shape of the function that is given by $f(x)$ that these values are indeed a minimum and not a maximum.
Plugging back in to the primitive function, we get $\mathcal{L}(x; w_+^*, w_-^*) =11.5e^{-1}- 12e^{-2} + 0.5e^2-6$.
Hence, for any smaller $\epsilon$ we do not have a network that can approximate to this error.

\subsection{Scaling relationship: Proof of Lemma~\ref{thm:scaling}}
\begin{theorem*}
Let  $h\left(\bm{\theta_0}, \bm{\sigma}\right)$ denote a transformation of the parameters $\bm{\theta_0}$ of the deep neural network $f_0$, where each weight is multiplied by a scalar $\sigma_l$, i.e., $h^{(l)}_{ij}(w^{(l)}_{0,ij}) = \sigma_l w^{(l)}_{0,ij}$, and each bias is transformed to $h^{(l)}_{i} (b^{(l)}_{0,i}) = \prod^l_{m=1}\sigma_m b^{(l)}_{0,i}$.
Then, we have 
$f\left(x \mid h(\bm{\theta_0}, \bm{\sigma})\right) = \prod^L_{l=1} \sigma_l f(x \mid \bm{\theta_0})$.
\end{theorem*}
\begin{proof}
Let the activation function $\phi$ of a neuron either be a ReLU $\phi(x) = \max(x,0)$ or the identity $\phi(x) = x$. A neuron $x^{(l)}_i$ in the original network becomes $g\left(x^{(l)}_i\right)$ after parameter transformation. 
We prove the statement by induction over the depth $L$ of a deep neural network. 

First, assume that $L=1$ so that we have $x^{(1)}_i = \phi\left(\sum_j w^{(1)}_{ij} x_j + b^{(1)}_i\right)$ After transformation by $w^{(1)}_{ij} \mapsto \sigma_1 w^{(1)}_{ij} $ and  $b^{(1)}_i \mapsto \sigma_1 b^{(1)}_i$, we receive $g\left(x^{(1)}_i\right) = \phi\left(\sum_j w^{(1)}_{ij} \sigma_1 x_j + \sigma_1 b^{(1)}_i\right) = \sigma_1 x^{(1)}_i$ because of the homogeneity of $\phi(\cdot)$.
This proves our claim for $L=1$.

Next, our induction hypothesis is that $g\left(x^{(L-1)}_i\right)  =  \prod^{L-1}_{m=1}\sigma_m x^{(L-1)}_i$.
It follows that 
\begin{align}
     g\left(x^{(L)}_i\right) & =  \phi\left(\sum_j w^{(L)}_{ij}\sigma_L g\left(x^{(L-1)}_j\right) + b^{(L)}_i \prod^{L}_{m=1} \sigma_m  \right) \ \ \ (\text{def. of transformation})\\
   &  = \phi\left(\sum_j w^{(L)}_{ij} \sigma_L  \prod^{L-1}_{m=1} \sigma_m  x^{(L-1)}_j +  b^{(L)}_i \prod^{L}_{m=1} \sigma_m  \right) \ \ \ (\text{induction hypothesis})\\
   & = \prod^{L}_{m=1}\sigma_m x^{(L)}_i \ \ \ (\text{homogeneity of } \phi),
\end{align}
which was to be shown.
\end{proof}

\subsection{Training of $f_0$ is feasible: Proof of Lemma~\ref{thm:train}}
\begin{theorem*}
Assume that the weights and biases of a fully-connected deep neural network $f$ are drawn independently from distributions that are symmetric around the origin $0$ with variances $\sigma^{2}_{w,l}$ or $\sigma^{2}_{b,l}$, respectively. Then, for every input $\bm{x_0}$, the second moment of the output is
\begin{align}
\begin{split}
 \mathbb{E}\left(||\bm{f(x_0)}||^2_2 \right) = & ||\mathbf{ x}^{(0)}||^2_2  \Pi^L_{l=1}\frac{n_l \sigma^{2}_{w,l}}{2}  + \sigma^{2}_{b,L} \frac{n_L}{2}  +   \sum^{L-1}_{l=1} \sigma^{2}_{b,l} \frac{n_l}{2} \Pi^{L}_{k=l+1}\frac{n_k\sigma^{2}_{w,k}}{2}.
\end{split}
\end{align}
\end{theorem*}
\begin{proof}
First, let us focus on the distribution of a neuron $x^{(l)}_i$ given all neurons of the previous layer with $x^{(l)}_i = \phi\left(h^{(l)}_i\right)$.
Since we assume that the weights and biases are distributed independently with zero mean, it follows that also the preactivation $h^{(l)}_i = \sum_j w^{(l)}_{ij} x^{(l-1)}_j + b^{(l)}_i$ has zero mean and variance $\mathbb{V}\left(h^{(l)}_i \mid \bm{x^{(l-1)}}\right) = \sum_j \left(x^{(l-1)}_j\right)^2 \mathbb{V}\left(w^{(l)}_{ij}\right) + \mathbb{V}\left(b^{(l)}_{i}\right) = \sigma^2_{w,l} ||\bm{x^{(l-1)}}||^2 + \sigma^2_{b,l}$, where $\mathbb{V}$ is the variance operator.
It is furthermore symmetric around zero so that a neuron $x^{(l)}_i = \phi\left(h^{(l)}_i\right) \sim 0.5 \delta_0 + 0.5 p_{h_{l,+}}$ is projected to zero with probability $0.5$ and otherwise follows the distribution of the positive preactivation $p_{h_{l,+}}$, where $\delta_0$ denotes the delta distribution at $0$ and $h_{l,+}$ the random variable $h_{l}$ conditional on $h_{l} > 0$.
In consequence, the squared neuron value $(x^{(l)}_i)^2 \sim 0.5 \delta_0 + 0.5 p_{h^2_{l}}$ has expectation $\mathbb{E}\left(\left(x^{(l)}_i\right)^2 \right) = 0.5 \; \mathbb{E}\left(\left(h^{(l)}_i\right)^2 \right) = 0.5 \; \sigma^2_{w,l} ||\bm{x^{(l-1)}}||^2 + 0.5 \; \sigma^2_{b,l}$.

Since all the neurons are independent and identically distributed given the neurons of the previous layer, we can easily deduce the expected signal norm 
\begin{align}
    \mathbb{E}\left(||\bm{x^{(l)}}||^2_2 ~\bigg|~ ||\bm{x^{(l-1)}}||^2_2 \right) = \sum^{n_l}_{i=1}\mathbb{E}\left( \left(x^{(l)}_i\right)^2 ~\bigg|~ ||\bm{x^{(l-1)}}||^2_2 \right)
    = \frac{n_l}{2} \left(\sigma^2_{w,l} ||\bm{x^{(l-1)}}||^2 + \sigma^2_{b,l} \right).
\end{align}

This gives us the expected signal norm of an arbitrary layer conditioned on the previous layer. 
We can use this relationship to also compute the average squared signal norm of the output layer, which is

\begin{align}
    \mathbb{E}\left(||\bm{f(x_0)}||^2_2 \right) = \mathbb{E}\left(||\bm{x^{(L)}}||^2_2\right) = \mathbb{E}\left(\mathbb{E}\left(||\bm{x^{(L)}}||^2_2 ~\bigg|~ ||\bm{x^{(1)}}||^2_2 \right)\right),
\end{align}
where the first equality is by definition of the network, and the second equality holds by law of total expectation. By recursively repeating this argument on the inner expectation, we get
\begin{align}
 \mathbb{E}\left(||\bm{f(x_0)}||^2_2 \right) &= \mathbb{E}\left(\mathbb{E}\left(||\bm{x^{(L)}}||^2_2  ~\bigg|~ ||\bm{x^{(1)}}||^2_2 \right)\right) \\
    &= \mathbb{E}\left(\mathbb{E}\left(\mathbb{E}\left(||\bm{x^{(L)}}||^2_2 ~\bigg|~  ||\bm{x^{(2)}}||^2_2 \right) ~\bigg|~ ||\bm{x^{(1)}}||^2_2 \right)\right) \\
    &= \mathbb{E}\left(\mathbb{E}\left(\ldots\mathbb{E}\left(\mathbb{E}\left(||\bm{x^{(L)}}||^2_2 ~\bigg|~  ||\bm{x^{(L-1)}}||^2_2 \right) ~\bigg|~  ||\bm{x^{(L-2)}}||^2_2 \right)\ldots ~\bigg|~ ||\bm{x^{(1)}}||^2_2 \right)\right).
\end{align}

Using the derivation further above, which provides a solution to the expected signal norm for a layer conditioned on the previous layer, we can iteratively resolve the innermost expectation
\begin{align}
\mathbb{E}\left(||\bm{f(x_0)}||^2_2 \right) &=\mathbb{E}\left(\mathbb{E}\left(\ldots\mathbb{E}\left(\mathbb{E}\left(||\bm{x^{(L)}}||^2_2 ~\bigg|~  ||\bm{x^{(L-1)}}||^2_2 \right) ~\bigg|~  ||\bm{x^{(L-2)}}||^2_2 \right)\ldots ~\bigg|~ ||\bm{x^{(1)}}||^2_2 \right) \right) \\
&= \mathbb{E}\left(\mathbb{E}\left(\ldots\mathbb{E}\left(
\frac{n_L}{2}\left(\sigma^2_{w,L}||\bm{x^{(L-1)}}||_2^2 + \sigma^2_{b,L} \right)
~\bigg|~  ||\bm{x^{(L-2)}}||^2_2 \right)\ldots ~\bigg|~ ||\bm{x^{(1)}}||^2_2 \right)\right)\\
&= \mathbb{E}\left(\mathbb{E}\left(\ldots\frac{n_L\sigma^2_{w,L}}{2}\mathbb{E}\left(
||\bm{x^{(L-1)}}||^2_2 ~\bigg|~  ||\bm{x^{(L-2)}}||^2_2 \right) + \frac{n_L\sigma^2_{b,L}}{2}\ldots ~\bigg|~ ||\bm{x^{(1)}}||^2_2 \right)\right).
\end{align}

Repeating this last argument provides the statement that was to be shown.

\end{proof}

\subsection{Error propagation: Proof of Lemma~\ref{thm:approx}}
\begin{theorem*}
Assume $\epsilon > 0$ and let the target network $f$ and its approximation $f_{\epsilon}$ have the same architecture.
If every parameter $\theta$ of $f$ and corresponding $\theta_{\epsilon}$ of $f_{\epsilon}$ in layer $l$ fulfils $|\theta_{\epsilon} - \theta| \leq \epsilon_{l}$ for 
\begin{align}
\epsilon_l := \epsilon \left(L \sqrt{m_{l}} \left(1+\sup_{x \in {[-1,1]}^{n_0}} ||\bm{x}^{(l-1)}||_{1}\right) \prod^L_{k=l+1} \left(||\bm{W}^{(l)}||_{\infty} + \epsilon/L \right) \right)^{-1},
\end{align}
then it follows that $||f-f_{\epsilon}||_{\infty} \leq \epsilon$. 
\end{theorem*}
\begin{proof}
Our objective is to bound $||f-f_{\epsilon}||_{\infty} \leq \epsilon$.
We frequently use the triangle inequality and that $|\phi(x)-\phi(y)| \leq |x-y|$ is Lipschitz continuous with Lipschitz constant $1$ to derive
\begin{align*}
||\bm{x^{(l)}} - \bm{x^{(l)}_{\epsilon}}||_2 & \leq ||\bm{h^{(l)}} - \bm{h^{(l)}_{\epsilon}}||_2 \\
& \leq  || \left(\bm{W^{(l)}} - \bm{W^{(l)}_{\epsilon}}\right) \bm{x^{(l-1)}} ||_2   + ||\bm{b^{(l)}}- \bm{b^{(l)}_{\epsilon}}||_2 + ||\bm{W^{(l)}_{\epsilon}} \left(\bm{x^{(l-1)}} - \bm{x^{(l-1)}_{\epsilon}}\right)||_2\\
    & \leq \epsilon_l \sqrt{m_{l}} \sup_{x \in {[-1,1]}^{n_0}} ||\bm{x}^{(l-1)}||_{1} + \epsilon_l \sqrt{m_{l}} + \left(||\bm{W^{(l)}}||_{\infty} + \epsilon_l \right) ||\left(\bm{x^{(l-1)}} - \bm{x^{(l-1)}_{\epsilon}}\right)||_2
\end{align*}
with $\epsilon_l \leq \epsilon/L$. 
$m_{l}$ denotes the number of parameters in layer $l$ that are smaller than $\epsilon_l$ and $||\bm{W}||_{\infty} =  \max_{i,j} |w_{i,j}|$.
Note that $m_{l} \leq n_l k_{l,\text{max}}$. 
The last inequality follows from the fact that all entries of the matrix $\left(\bm{W^{(l)}} - \bm{W^{(l)}_{\epsilon}}\right)$ and of the vector $(\bm{b^{(l)}}- \bm{b^{(l)}_{\epsilon}})$ are bounded by $\epsilon_l$ and maximally $m_l$ of these entries are nonzero.
Furthermore, $||\bm{W^{(l)}_{\epsilon}}||_{\infty} \leq \left(||\bm{W^{(l)}}||_{\infty} + \epsilon_l \right)$ follows again from the fact that each entry of $\left(\bm{W^{(l)}} - \bm{W^{(l)}_{\epsilon}}\right)$ is bounded by $\epsilon_l$.

Thus, at the last layer it holds for all $\bm{x} \in [-1,1]^{n_0}$ that
\begin{align*}
    ||f(x)-f_{\epsilon}(x)||_2 & = ||\bm{x^{(L)}} - \bm{x^{(L)}_{\epsilon}}||_2 \\
    & \leq \sum^L_{l=1} \epsilon_l \sqrt{m_{l}} \left(1+\sup_{x \in {[-1,1]}^{n_0}} ||\bm{x}^{(l-1)}||_{1}\right) \prod^L_{k=l+1} \left(||\bm{W}^{(l)}||_{\infty} + \epsilon/L \right) \leq L \frac{\epsilon}{L} = \epsilon,
\end{align*}
using the definition of $\epsilon_l$ in the last step.
\end{proof}

\subsection{Existence of lottery ticket: Proof of Theorem~\ref{thm:LTexist2}}
\begin{theorem*}
Assume that ${\epsilon, \delta \in {(0,1)}}$ and a target network $f$ with depth $L$ and architecture $\bar{n}$ are given. 
Each weight and bias of a larger deep neural network $f_{0}$ with depth $2L$ and architecture $\bar{n}_0$ is initialized independently, uniformly at random according to $w^{(l)}_{ij} \sim U{([-\sigma_{w,l}, \sigma_{w,l}])}$ and $b^{(l)}_{i} \sim U{([-\prod^l_{k=1}\sigma_{w,k}, \prod^l_{k=1}\sigma_{w,k}])}$. Then, with probability at least $1-\delta$, $f_{0}$ contains an approximation $f_{\epsilon} \subset f_{0}$ so that $||f- \lambda f_{\epsilon}||_{\infty} \leq \epsilon$ if for l = 1, ..., L
\begin{align}
n_{2l-1,0} = C n_{l-1} \log\left(\frac{1}{\min\left\{\epsilon_l, \delta_l\right\}}\right)
  \text{ and } n_{2l,0} = n_l,
\end{align}
where $\epsilon_l$ is defined in Eq.~(\ref{eq:epsTheta}), $\delta_l = \delta/(L k_{l-1,\text{max}} n_{l})$, and the output is scaled by $\lambda = \prod^{2L}_{l=1}\sigma^{-1}_{w,l}$.
\end{theorem*}
\begin{proof}
Lemma~\ref{thm:scaling} simplifies the above parameter initialization to an equivalent setting, in which each parameter is distributed as $w_{ij}, b_i \sim {U[-1,1]}$, while the overall output is scaled by the stated scaling factor $\lambda$. 
We assume that all parameters are bounded by $1$ so that we can find them within the range $[-1,1]$. 
Otherwise, we would need to increase $n_{2l-1}$ by a factor that is proportional to the maximum parameter value $\theta_{\text{max}}$, which is integrated into our constant $C$. 

Every layer $l$ of $f$ corresponds in our construction to two layers of $f_0$, i.e., layers $2l-1$ and $2l$.
The neurons in layer $2l$ correspond directly to the output neurons in layer $l$ of $f$.
Thus, we only need width $n_{2l,0} = n_l$ in $f_0$. 
Layer $2l-1$ serves the construction of intermediary neurons of in-degree $1$.
Using the identity $\phi(x) = \phi(x) - \phi(-x)$, we see that all neurons in layer $l$ of $f$ can indeed be represented by a two-layer neural network consisting of $3 n_{l-1}$ intermediary neurons of degree $1$, as
\begin{align}
    x^{(l)}_i & = \phi\left(\sum_j w^{(l)}_{ij} x^{(l-1)}_j + b^{(l)}_i\right) \\
    & = \phi\left(\sum_j w^{(l)}_{ij} \phi\left(x^{(l-1)}_j\right) -  \sum_j w^{(l)}_{ij} \phi\left(-x^{(l-1)}_j\right) +  b^{(l)}_i \phi\left( 1 \right)  \right).
\end{align}
According to Lemma~\ref{thm:approx}, we need to approximate each $w^{(l)}_{ij}$ and $-w^{(l)}_{ij}$ up to error $\epsilon_l/2$ and $b^{(l)}_i$ up to error $\epsilon_l$ to guarantee our overall approximation objective. 
Since we have to do this for every parameter, our overall approximation can only be successful with probability $1-\delta$ if we increase our success probability for each parameter, $1-\delta_l$, accordingly.
In total, we have $m_{l,\text{max}}$ of such nonzero parameters in layer $l$ with $m_{l,\text{max}} \leq 2 n_l k_{l,\text{max}}$. 
(To be precise, $m_{l,\text{max}}$ denotes the number of parameters that are bigger than $\epsilon_l$).
The successes of finding different parameters are not necessarily independent but we can identify a sufficient $\delta_l$ with the help of a union bound. 
Accordingly, $1-\delta \geq 1 - \sum^L_{l=1} \delta_l m_{l,\text{max}}$ is fulfilled for $\delta_l \leq \delta/(2 L n_l k_{l,\text{max}})$. 
Note that we later integrate the factor $2$ in the constant related to the layer width.

With probability at least $1-\delta_l$, we can approximate each single parameter by solving the subset sum problem for the corresponding neuron. 
As outlined in Cor.~2 by \citet{pensia2020optimal}, which is based on Cor.~3.3 by \citet{subsetsum}, we need $C \log\left(\frac{1}{\min\left(\delta_l, \epsilon_l\right)} \right)$
neurons in layer $2l-1$ per neuron of the form $\phi\left(\pm x^{(l-1)}_j\right)$ or $\phi\left(1\right)$.
Since we have to represent $3 n_{l-1}$ of these neurons, in total we require layer $2l-1$ of $f_0$ to have width 
\begin{align}
n_{2l-1,0} \geq C n_{l-1} \log\left(\frac{1}{\min\left(\delta_l, \epsilon_l\right)} \right).    
\end{align}
Next, we briefly explain the main ideas that lead to this result.
The main difference of our situations in comparison with \citet{pensia2020optimal} is that we additionally create neurons of the form $\phi\left(1\right)$ to represent nonzero biases. 
Let $\phi(y)$ be our target neuron, where $y$ is either $y = x^{(l-1)}_j$ or $y=1$ depending on which neuron we want to represent. 
Note that we can construct multiple candidates for a neuron $\phi(y)$ by pruning neurons in layer $2l-1$. 
We achieve that by setting all weights that do not lead to $y$ in the previous layer and the bias term of a neuron to zero or, if $y=1$, we set all weights to zero and keep the nonzero bias term if the bias is positive.
Let the index set of the such pruned neurons corresponding to $y$ be $I$. 
This leaves us with  neurons of the form $w_{2,i}\phi\left(w_{1,i} y\right)$ with  $\text{sign}(y) w_{1,i} \sim U[0,1]$ and $w_{2,i} \sim U[-1,1]$ for $i \in I$. 
For the probability distribution of $w_{1,i} w_{2,i}$, Cor.~2 of \citet{pensia2020optimal} states that it contains a uniform distribution.
It follow that the subset sum problem has a solution. 
Thus, for any parameter $\theta \in [-1,1]$ there exists a subset $S \subset I$ so that with probability at least $1-\delta_l$
\begin{align}
    |\theta - \sum_{i \in S} w_{1,i} w_{2,i}| \leq \epsilon_l.
\end{align}
if $|I| \geq C \log\left(\frac{1}{\min\left(\delta_l, \epsilon_l\right)} \right)$, which was to be shown.

Note that the same result also holds for normally distributed $w_{1,i}$, $w_{2,i}$, as their product contains a uniform distribution.
This follows from the fact that the normal distribution contains a uniform distribution \citet{pensia2020optimal}.
Thus, the product of two normal distributions contains a product of two uniform distributions and this product of uniform distributions contains a uniform distribution as stated by Cor.~2 of \citet{pensia2020optimal}.
\end{proof}

`Looks-linear' initializations are also covered by this proof. 
When we can construct a parameter $w^{(l)}_{ij}$ by solving a subset sum problem, we can construct $-w^{(l)}_{ij}$ in the same way just with the negative correspondents of the parameters that construct $w^{(l)}_{ij}$. 

\subsection{Results of proof construction tickets}

To validate that the constructed SLTs from our proof do as well as a target network for our proposed nonzero initialization schemes,
we followed the construction of mother network $f_0$ and approximation $f_{\epsilon}$  of the existence proof. As target network $f$ we considered a lottery ticket found by iterative magnitude pruning in a LeNet on MNIST, as detailed in the main text.
We report the accuracy, number of parameters and sparsity with respect to the mother network for different initialization schemes in Tab.~\ref{table:const_results}.

\begin{table}
\centering
\begin{tabular}{ r |  l  c  r  c} 
 \toprule
  & Initialization & \% Acc.& \# Param. & Sparsity\\
 \midrule
 Target $f$ & He & 97.96 & 18697 &\\ 
 Appr. $f_{\epsilon}$ & Orthogonal & 97.98 & 106192 & 0.0027\\
 Appr. $f_{\epsilon}$ & Uniform & 97.94 & 112157  & 0.0029\\
 Appr. $f_{\epsilon}$ & Normal & 97.96 & 121153  & 0.0031\\
 \bottomrule
\end{tabular}
\caption{\textit{Constructed target networks.} Results of constructing lottery tickets to a pruned target network from LeNet on MNIST.
Sparsity refers to the sparsity of approximations $f_{\epsilon}$ with respect to the mother network $f_0$.}
\label{table:const_results}
\end{table}

\subsection{Algorithm for \texttt{edge-popup-scaled}}
In this section we outline the proposed \texttt{edge-popup-scaled} algorithm. 
\begin{algorithm}[h!]
   \caption{\texttt{edge-popup-scaled}}
   \label{alg:edgepopup_scaled}
   \KwIn{Data $(\textbf{X},\bm{y})$, Sparsity $\rho$, levels $e_a$, Epochs $T$, Mother Network $f_0$ with depth $L$, Loss: $loss(.)$}
   {\bfseries Initialize:} Parameters $ \bm{W}_0, \bm{b}_0, w^{(l)}_{ij} \sim U{([-\sigma_{w,l}, \sigma_{w,l}])}$ and $b^{(l)}_{i} \sim U{([-\prod^l_{k=1}\sigma_{w,k},
   \prod^l_{k=1}\sigma_{w,k}])}$, \\
   Scores $ \bm{S}, s_{w_{ij}^{(l)}} \sim $ Const. $(0.5)$ and $s_{b_{i}^{(l)}} \sim$ Const. $(0.5)$\\
   \For{$i=1$ {\bfseries to} $e_a$}{
     sparsity $= \rho^{i / e_a}$\\
     \For{$t \text{ in } [1,..,T]$}{
        Retain weights with largest score in the mask $\bm{M}$\\
        $\bm{M} = getMask(f_0, \bm{S}, \text{sparsity})$ \\
        $\mathcal{L} = loss(f(\bm{x}_i, \bm{M} \cdot [\bm{W}_0,\bm{b}_0]), y_i)$\\
        $\tilde{s}_{w_{ij}^{(l)}} = s_{w_{ij}^{(l)}} - \alpha  \frac{\partial \mathcal{L}}{\partial \bm{h}^{(l-1)}}w_{ij}^{(l)}\phi(\bm{h}_i^{(l-1)})$ \\
        {\bfseries scaling: }\\
        $\tilde{\lambda} = \argmin_{\lambda > 0} \mathcal{L}(y, \lambda x^{(L)}) $   \\
        $w^{(l)}_{ij} = w^{(l)}_{ij} \tilde{\lambda}^{1/L}, \ \ \  b^{(l)}_{i} = b^{(l)}_{i} \tilde{\lambda}^{l/L}$
   }
   }
   $\lambda_{opt} = \tilde{\lambda}, \bm{S}_{opt} = \tilde{\bm{S}}, \bm{M}_{opt} = getMask(f_0, \bm{S}_{opt}, \rho)$\\
   $f_{\eps} = f_0(\bm{M}_{opt}\cdot[\bm{W}_0, \bm{b}_0] )$\\
   {\bfseries output: } $f_{\eps}, \lambda_{opt}$\\
\end{algorithm}
Here, the $getMask(.)$ function returns a binary mask $M$, where the weights corresponding to the largest scores (according to the sparsity $\rho$) are retained in the mask.

\section{Experiments}

\subsection{Experiments on imaging data}
\label{app:cifar}

We conducted experiments on the benchmark CIFAR10 for classification with CNNs and reconstruction with autoencoders.
In particular, we pruned He initialized VGG16 models for classification with \texttt{edge-popup} with annealing for $5$ levels, $50$ epochs, and  optimized via SGD with momentum of $.9$, learning rate of $.1$, and cosine annealing of the learning rate, as discussed in the original paper~\cite{edgepopup}. Experiments were repeated $5$ times. 
For reconstruction, we pruned a convolutional autoencoder as detailed in Table. \ref{table:ae_model}, with annealing for $5$ levels, $50$ epochs, optimized with Adam and a constant learning rate of $.001$. These set of experiments were repeated $4$ times.  
From the results, reported in Fig.~\ref{fig:imaging}, our first observation is that we can successfully prune for strong lottery tickets for both zero bias as well as nonzero bias initialized networks.
Second, the achieved results do not differ significantly. Upon further investigation of the autoencoder tickets, this might be the case due to insufficient pruning of bias parameters, as most of the layers in the ticket do not contain any bias. 
This could be a systematic property of image data, i.e. convolutions do not require bias parameters for efficient prediction or reconstruction, or could be a shortcoming of current strong lottery ticket pruning to be parameter-inefficient and hence unable to draw advantage from the bias parameters.
In recent literature, it was shown that \texttt{edge-popup} is indeed unable to recover a ground truth ticket of high sparsity for CIFAR10~\cite{fischer2022planting}.
Only once more efficient SLT pruners are available, we can give a definite answer to the question of whether biases are necessary for efficent SLTs for imaging data.

\begin{figure}
  \centering
    \begin{subfigure}[t]{0.99\textwidth}
    \centering
		\ifpdf
		\begin{tikzpicture}
		\begin{axis}[
		jonas line,
		cycle list name=juarez4,
		xmode=log,
		width = 10cm,
		height = 4cm,
		xlabel = {Sparsity}, 
		ylabel= {Accuracy},
		xmin=0.01, xmax=0.55,
		ymin=0.3, ymax=1,
		x label style 		= {at={(axis description cs:0.5,-0.1)}, anchor=north, font=\scriptsize},
		y label style 		= {at={(axis description cs:0.0,0.6)},  anchor=south, font=\scriptsize},
		legend style={nodes={scale=0.9, transform shape}, at={(0.98,0.02)}, anchor=south east, row sep=-1.4pt, font=\tiny}
		]
		
		\addplot+[forget plot,eda errorbarcolored, y dir=plus, y explicit]
		table[x=sparsity, y=mean, y error=max] {results_kn_nonzerobias.txt};
		\addplot+[eda errorbarcolored, y dir=minus, y explicit]
		table[x=sparsity, y=mean, y error=min] {results_kn_nonzerobias.txt};
        \addlegendentry{nonzero bias}
        
        \addplot+[forget plot,eda errorbarcolored, y dir=plus, y explicit]
		table[x expr=\thisrow{sparsity}*1.025, y=mean, y error=max] {results_kn_zerobias.txt};
		\addplot+[eda errorbarcolored, y dir=minus, y explicit]
		table[x expr=\thisrow{sparsity}*1.025, y=mean, y error=min] {results_kn_zerobias.txt};
        \addlegendentry{zero bias}
		
		\end{axis}
		\end{tikzpicture}
		\fi
		\caption{VGG16 results.}
    \end{subfigure}
    \centering
    \begin{subfigure}[t]{0.99\textwidth}
    \centering
		\ifpdf
		\begin{tikzpicture}
		\begin{axis}[
		jonas line,
		cycle list name=juarez4,
		xmode=log,
		ymode=log,
		width = 10cm,
		height = 4cm,
		xlabel = {Sparsity}, 
		ylabel= {MSE},
		xmin=0.01, xmax=0.11,
		ymin=0.01, ymax=1,
		x label style 		= {at={(axis description cs:0.5,-0.1)}, anchor=north, font=\scriptsize},
		y label style 		= {at={(axis description cs:0.0,0.6)},  anchor=south, font=\scriptsize},
		legend style={nodes={scale=0.9, transform shape}, at={(0.98,0.02)}, anchor=south east, row sep=-1.4pt, font=\tiny}
		]
		
		\addplot+[forget plot,eda errorbarcolored, y dir=plus, y explicit]
		table[x=sparsity, y=mean, y error=max] {results_kn-nonzero-bias.txt};
		\addplot+[eda errorbarcolored, y dir=minus, y explicit]
		table[x=sparsity, y=mean, y error=min] {results_kn-nonzero-bias.txt};
        
        \addplot+[forget plot,eda errorbarcolored, y dir=plus, y explicit]
		table[x expr=\thisrow{sparsity}*1.025, y=mean, y error=max] {results_kn-zero-bias.txt};
		\addplot+[eda errorbarcolored, y dir=minus, y explicit]
		table[x expr=\thisrow{sparsity}*1.025, y=mean, y error=min] {results_kn-zero-bias.txt};
		
		\end{axis}
		\end{tikzpicture}
		\fi
		\caption{Autoencoder results.}
    \end{subfigure}
    \caption{\textit{CIFAR10 results}. Comparison of \texttt{edge-popup} results on CIFAR10 with models with zero vs nonzero bias initializations. Reported are mean as well as minimum and maximum obtained from 5 repetitions for classifiation (a) and reconstruction (b) across different sparsity levels.}
    \label{fig:imaging}
\end{figure}

\begin{table}[h!]
\centering
\begin{tabular}{ r |  l  c } 
 \toprule
  Layer name & Description\\
 \midrule
 Encoder 1 & Conv(3, 64, 3) \\ 
 Encoder 2 & Conv(64, 64, 3) \\ 
 Encoder 3 & Conv(64, 128, 3) \\ 
 Encoder 4 & Conv(128, 128, 3) \\ 
 Encoder 5 & Conv(128, 128, 3) \\ 
 Encoder 6 & Linear(2048, 512) \\ 
 Decoder 1 & Linear(512, 2048) \\ 
 Decoder 2 & ConvTranspose(128, 128, 3)\\ 
 Decoder 3 & Conv(128, 128, 3) \\ 
 Decoder 4 & ConvTranspose(128, 64, 3) \\ 
 Decoder 5 & Conv(64, 64, 3) \\ 
 Decoder 6 & ConvTranspose(128, 3, 3) \\ 
 \bottomrule
\end{tabular}
\caption{\textit{Autoencoder architecture}. Description of each convolutional layer of the autoencoder. Each layer is followed by a ReLU activation, the final layer has a Tanh activation}
\label{table:ae_model}
\end{table}

\end{document}